%% file: main.tex
\documentclass[twoside]{article}

\usepackage[accepted]{aistats2025}
\usepackage[round]{natbib}

\input{macros}
\begin{document}

\runningtitle{Robust Multiple-Source Offline Policy Learning}


\twocolumn[

\aistatstitle{Robust Offline Policy Learning with \\ Observational Data from Multiple Sources}

\aistatsauthor{ Aldo Gael Carranza \And Susan Athey }

\aistatsaddress{ Stanford University \And  Stanford University } ]

% TL;DR
% We propose a minimax regret optimization framework for learning decision policies from bandit feedback data across multiple data sources, ensuring robust generalization through uniformly low regret across various target distributions.

\begin{abstract}
We consider the problem of using observational bandit feedback data from multiple heterogeneous data sources to learn a personalized decision policy that robustly generalizes across diverse target settings. To achieve this, we propose a minimax regret optimization objective to ensure uniformly low regret under general mixtures of the source distributions. We develop a policy learning algorithm tailored to this objective, combining doubly robust offline policy evaluation techniques and no-regret learning algorithms for minimax optimization. Our regret analysis shows that this approach achieves the minimal worst-case mixture regret up to a moderated vanishing rate of the total data across all sources. Our analysis, extensions, and experimental results demonstrate the benefits of this approach for learning robust decision policies from multiple data sources.
\end{abstract}

\section{INTRODUCTION}

Offline policy learning from observational bandit feedback data is an effective approach for learning personalized decision policies in applications where obtaining real-time data is impractical \citep{swaminathan2015batch, kitagawa2018should, athey2021policy}. Typically, the observational data used in offline policy learning is assumed to originate from a single source. However, in practice, multiple datasets collected from various experiments under different populations, environments, or logging policies are often available \citep{kallus2021optimal}. For example, a healthcare policymaker aiming to design a targeted medical intervention policy might have access to data from various hospitals, each having conducted different clinical trials on distinct patient populations. Leveraging heterogeneous observational datasets effectively, with their broader and more diverse coverage of the decision space, can lead to more generalizable policies across a wider range of settings.

However, simply training a policy on aggregated data does not guarantee robust performance. While such a policy may perform well on the uniform source mixture distribution, it may fail entirely in the original individual source settings or similar environments, particularly when significant distribution shifts exist across sources. This undermines the goal of using multiple observational datasets to generalize effectively across settings captured by any of the source distributions. Developing a policy that performs well across diverse target distributions requires robust domain adaptation. The key question is: given a family of target distributions that can be modeled by the source distributions or combinations thereof, can we train a policy that reliably generalizes across this entire family?

In this work, we address this challenge by framing the problem as one of minimax regret optimization over a family of target distributions represented as mixtures of the source distributions. We propose a novel offline policy learning algorithm that integrates techniques from doubly robust policy evaluation, offline policy optimization oracles, and no-regret learning algorithms for minimax optimization. Our analysis establishes finite-sample regret bounds, showing that the robustly trained policy achieves the minimal worst-case regret across the family of target distributions up to a vanishing rate of the total data across all sources. Additionally, we extend these bounds to target distributions that are not fully captured by any mixture of the source distributions. We characterize these regret bounds in terms of source heterogeneity and source distribution shift. Our theoretical analysis and experimental validation demonstrate the benefits of multi-source adaptation for robust offline policy learning across diverse environments.

\section{RELATED WORK}

\textit{Offline Policy Learning.}
There have been many recent advancements in offline policy learning from observational bandit feedback data. \citet{swaminathan2015batch, kitagawa2018should} introduced foundational frameworks for learning structured decision policies via offline policy evaluation strategies. \citet{athey2021policy} achieved optimal regret rates under unknown propensities through doubly robust estimators, while \citet{zhou2023offline} extended these results to the multi-action setting. \citet{kallus2018balanced} directly derived optimal weights for target policies from the data, and \citet{zhan2021policy} ensured optimal regret guarantees under adaptively collected data with diminishing propensities. \citet{jin2022policy} relaxed the uniform overlap assumption, allowing for partial overlap under the optimal policy. We also mention that many contextual bandit methods often utilize offline policy learning oracles when devising adaptive action-assignment rules \citep{bietti2021contextual,simchi2022bypassing, carranza2022flexible}. Of particular relevance to our work on policy learning with heterogeneous data sources, \citet{agarwal2017effective, he2019off, kallus2021optimal} leveraged data from multiple historical logging policies, though their setting assumes the same underlying populations and environments. Our problem setting closely relates to the problem of multi-task offline policy learning \citep{hong2023multi}, where each task corresponds to a different data source. 

\textit{Multiple Source Adaptation.} \citet{mohri2019agnostic} introduced a framework for robust supervised learning across multiple sources, presenting concepts like weighted Rademacher complexity and skewness measures, which are extensively utilized in our work.

\textit{Minimax Excess Risk Optimization.} \citet{agarwal2022minimax} proposed minimax excess risk minimization in the context of distributionally robust supervised learning, aiming to achieve uniformly low regret across a family of test distributions. We adapt analytical and algorithmic techniques from this framework for minimax regret optimization in robust multi-source offline policy learning for a family of target distributions.

\section{PRELIMINARIES}\label{sec:Preliminaries}

\subsection{Setting}\label{sec:Preliminaries-Setting}

We introduce the setting of offline policy learning from observational bandit feedback data across multiple data sources. Let $\calX\subset\R^p$ be the context space, $\calA=\{a_1,\dots,a_d\}$ be the finite action space with $d$ actions, and $\calY\subset\R$ be the reward space. A \textit{decision policy} $\pi:\calX\to\calA$ is a deterministic mapping from the context space $\calX$ to actions $\calA$. We assume there is a finite set of data sources $\sourceSet$, with each source $\s\in\sourceSet$ possessing a \textit{data-generating distribution} $\Di$ defined over $\calX\times\calY^d$ which governs how the source contexts $\Xi$ and potential reward outcomes $\Yi(a_1),\dots,\Yi(a_d)$ are generated. These source distributions may all be different from each other.

We assume access to observational bandit feedback data from each of the data sources. The aim is to use this data learn a policy that performs well uniformly across target distributions that can be captured by the source distributions. We consider the natural scenario where the target distribution can be modeled as a \textit{mixture} of the source distributions, i.e., $\calDw\coloneqq\sumiM\wi\calDc$ for some unknown mixture weights $\w$. As an extension, we will account for the case where there may be a \textit{discrepancy} between the target distribution and the mixture distributions. Since the the right choice of mixture weights is unknown, we must come up with a solution that is favorable for any weights in a specified set of \textit{valid mixture weights} in the simplex over the source set, i.e., $\w\in\W\subset\Delta(\sourceSet)$. In the following section, we introduce policy performance measures that capture these objectives.

\subsection{Objective}

First, we define the expected reward of a decision policy gained under a given source mixture distribution.

\begin{definition}[Mixture Policy Value]
    For any policy $\pi$, the \textit{mixture policy value} under the mixture weights $\w\in\W$ is
    \begin{equation*}
        \Qw(\pi)\coloneqq\opsE_{Z\sim\calDw}\regbra{Y(\pi(X))},
    \end{equation*}
    where the expectation is taken with respect to the mixture distribution $Z=(X,Y(a_1),\dots,Y(a_d))\sim\calDw$.
\end{definition}

The performance of a policy is typically characterized by the notion of regret against an optimal policy in a specified \textit{policy class} $\Pi\subset\{\pi:\calX\to\calA\}$, which we assume to be fixed throughout the paper.

\begin{definition}[Mixture Regret]
    For any policy $\pi$, the \textit{mixture regret} under the mixture weights $\w\in\W$ relative to the given policy class $\Pi$
    is
    \begin{equation*}
        \Rw(\pi)\coloneqq\max_{\pi'\in\Pi}Q_\w(\pi')-Q_\w(\pi).
    \end{equation*}
\end{definition}

Thus, the objective is to determine a policy in the specified policy class that minimizes the \textit{worst-case mixture regret} under the specified set of valid mixture weights. In this paper, we propose a policy learning procedure that approximately achieves the minimal worst-case mixture regret for any choice of source mixture $\wp\in\W$ up to a rate that vanishes with the total data, i.e, a procedure that learns a policy $\pihat\in\Pi$ that achieves a bound of the form
\begin{equation*}
    \Rwp(\pihat)\lesssim\minPi\maxwW\Rw(\pi) + \xi(n)
\end{equation*}
for some vanishing rate $\xi(n)$ of the total data size $n$ across all sources. We will then discuss extensions to regret bounds under general target distributions not captured by the class of mixture distributions.

\subsection{Data-Generating Processes}\label{sec:Preliminaries-Data}

We assume each source $\s\in\sourceSet$ has a \textit{local observational data set}
$\{(\Xij,\Aij,\Yij)\}_{i=1}^{\nc}\subset\calX\times\calA\times\calY$ consisting of $\nc\in\N$ triples of contexts, actions, and rewards collected using a \textit{local experimental stochastic policy} $\ei:\calX\to\Delta(\calA)$ in the following manner. For the $i$-th data point of source $\s\in\sourceSet$,
\begin{compactenum}
    \item nature samples $(\Xij,\Yij(a_1),\dots,\Yij(a_d))\sim\calDc$;

    \item source $\s$ is assigned action $\Aij\sim\ei(\cdot|\Xij)$;
    
    \item source $\s$ observes $\Yij=\Yij(\Aij)$ ;

    \item source $\s$ logs the data tuple $(\Xij,\Aij,\Yij)$.\footnote{If the propensity $\ei(\Aij|\Xij)=\P_{\calDcec}\!(\Aij|\Xij)$ is known, it also locally logged as it can facilitate subsequent policy value estimation.}
\end{compactenum}

\vspace{.5em}
\begin{remark}
We will let $\smash[b]{n\coloneqq\sumiM\ni}$ denote the \textit{total sample size} across sources, and we will derive regret bounds that scale with the total sample size.
\end{remark}

Note that although the counterfactual reward outcomes $\Yij(a)$ for all $a\in\calA\backslash\{\Aij\}$ exist in the source data-generating process, they are not observed in the realized data. All sources only observe the outcomes associated to their assigned treatments. For this reason, such observational data is also referred to as \textit{bandit feedback data} \citep{swaminathan2015batch}.

Given these data-generating processes, it will also be useful to introduce the induced data-generating distributions that incorporate how actions are sampled. For each source $\s\in\sourceSet$, the local historical policy $\ei$ induces a \textit{complete data-generating distribution} $\calDcec$ defined over $\calX\times\calA\times\calY^d$ that dictates how the entire local contexts, actions, and potential outcomes were sampled in the local data-generating process, i.e., $(\Xij,\Aij,\Yij(a_1),\dots,\Yij(a_d))\sim\calDcec$.

\subsection{Data Assumptions}\label{sec:Preliminaries-DataAssumptions}
We make the following standard assumptions on the data-generating process of any given source.
\begin{assumption}[Local Ignorability]\label{ass:dgp}
    For any source $\s\in\sourceSet$, the local complete data-generating distribution $(\Xi,\Ai,\Yi(a_1),\dots,\Yi(a_d))\sim\calDcec$ satisfies:
    \begin{compactenum}
        \item[(a)] \textit{Boundedness}: The marginal distribution of $\calDcec$ on the set of potential outcomes $\calY^d$ has a bounded support, i.e., there exists some $B_\s>0$ such that $\abs{\Yi(a)}\le B_\s$ for all $a\in\calA$.
        
        \item[(b)] \textit{Unconfoundedness}:
        Potential outcomes are independent of the observed action given the observed context, i.e., $(\Yi(a_1),\dots,\Yi(a_d))\indep\Ai\mid\Xi$.

        \item[(c)] \textit{Overlap}: For any given context, every action has a non-zero probability of being sampled, i.e., there exists some $\eta_\s>0$ such that $\P(\Ai=a|\Xi=x)\ge\eta_\s$ for any $a\in\calA$ and $x\in\calX$.
    \end{compactenum}
\end{assumption}

Note that the \textit{boundedness} assumption is not essential and we only impose it for simplicity in our analysis. With additional effort, we can instead rely on light-tail distributional assumptions such as sub-Gaussian potential outcomes as in \citep{athey2021policy}. \textit{Unconfoundedness} ensures that action assignment is as good as random after accounting for measured covariates, and it is necessary to ensure valid policy value estimation using inverse propensity-weighted strategies. The \textit{uniform overlap} condition ensures that the decision space is sufficiently explored to guarantee accurate evaluation of any policy. However, we note that this assumption may not be entirely necessary as recent work \citep{jin2022policy} has introduced a pessimism-based approach that does away with the uniform overlap assumption for all actions and only relies on overlap for the optimal policy, assuming the behavior policy from the data collection process is known.
Nevertheless, we decided to impose this uniform overlap assumption since pessimistic offline policy learning still
requires known propensities, whereas we consider a more general setting where propensities can be estimated. Moreover, we made the above assumptions to simplify our analysis and maintain the focus of our contributions on policy learning under multiple sources. In any case, these stated assumptions are standard and they are satisfied in many practical settings such as randomized controlled trials or A/B tests.

Next, we also impose the following local data scaling assumption on each source. 
\begin{assumption}[Local Data Scaling]\label{ass:LocalDataSizeScaling}
    All local sample sizes asymptotically increase with the total sample size, i.e., for each $\s\in\sourceSet$, $\nc=\Omega(\nu_\s(n))$ where $\nu_\s$ is an increasing function.
\end{assumption}
This assumption states that, asymptotically, the total sample size cannot increase without increasing across all data sources. We emphasize that this assumption is quite benign since $\nu_\s$ could be any slowly increasing function (e.g., an iterated logarithm) and the asymptotic lower bound condition even allows step-wise increments. We only impose this assumption to ensure that the regret bounds in our analysis scale with respect to the total sample size with sensible constants. However, it does come at the cost of excluding scenarios in which a source always contributes $O(1)$ amount of data relative to the total data, no matter how much more total data is made available in aggregate, in which case it may be better to exclude any such source, assuming no additional information.

\section{APPROACH}\label{sec:Approach}

Our general approach is to use the available observational data across sources to construct an appropriate estimator of the mixture regret and use this estimator to inform an appropriate optimization objective for determining a robust policy under any mixture weights.

\subsection{Nuisance Parameters}
We define the following functions which are referred to as \textit{nuisance parameters} since they are required to be separately known or estimated for the policy value estimates.

\begin{definition}[Nuisance Parameters]
    For any source $s\in\sourceSet$, their \textit{conditional response} function $\mui$ and \textit{inverse conditional propensity} function $\oi$ are defined, respectively, for any $x\in\calX$ and $a\in\calA$, as
    \begin{gather*}
        \mui(x;a)\coloneqq\E_{\calDcec}[\Yi(a)|\Xi=x], \\
        \oi(x;a)\coloneqq1/\P_{\calDcec}\!(\Ai=a|\Xi=x).
    \end{gather*}
    For convenience, we denote $\mui(x)=(\mui(x;a))_{a\in\calA}$ and $\oi(x)=(\oi(x;a))_{a\in\calA}$.
\end{definition}
In our estimation strategy, we must separately estimate the source conditional response and inverse conditional propensity functions when they are unknown. However, we can also pool data across sources known to be equivalent for improved nuisance parameter estimation.
Following the literature on double machine learning \cite{chernozhukov2018double}, we make the following high-level assumption on the estimators of these source nuisance parameters.
\begin{assumption}\label{ass:FiniteSampleError}
    For any source $\s\in\sourceSet$, the source estimates $\muhati$ and $\ohati$ of the nuisance parameters $\mui$ and $\oi$, respectively, trained on $\nc$ source data points satisfy
    the following squared error bounds:
    \begin{gather*}
        \E_{\calDc}\!\bigbra{\norm{\muhati(\Xi)-\mui(\Xi)}_2^2}\le\frac{o(1)}{\sqrt{\nc}}, \\
        \E_{\calDc}\!\bigbra{\norm{\ohati(\Xi)-\oi(\Xi)}_2^2}\le\frac{o(1)}{\sqrt{\nc}}.
    \end{gather*}
\end{assumption}

We emphasize this is a standard assumption in the double machine learning literature, and we can easily construct estimators that satisfy these rate conditions, given sufficient regularity on the nuisance parameters \citep{zhou2023offline}. See Appendix~\ref{app:NuisanceParameterEstimation} for more details. They can be estimated with widely available out-of-the-box regression and classification implementations.

\subsection{Policy Value and Regret Estimators}\label{sec:Approach-PolicyValueEstimators}

Next, we define our policy value estimators. For any $\s\in\sourceSet$, consider the \textit{source augmented inverse propensity weighted} (AIPW) score for each $a\in\calA$ to be
\begin{equation*}
    \Gi(a)\coloneqq\mui(\Xi;a)+\widebar{Y}^\s\cdot\oi(\Xi;a)\cdot\ones\{\Ai=a\},
\end{equation*}
where
\begin{math}
\widebar{Y}^\s=\Yi(\Ai)-\mui(\Xi;a)
\end{math}
are centered outcomes and $(\Xi,\Ai,\Yi(a_1),\dots,\Yi(a_d))\sim\calDcec$. One can readily show that this is an unbiased estimate of the mixture policy value, i.e., $\Qw(\pi)=\opsE_{\w}\opsE_{\calDcec}[\Gi(\pi(\Xi))]$ (see the proof in Lemma~\ref{lem:ExpectedOracleEqualsLocalPolicyValue}). Accordingly, our procedure is to estimate the source AIPW scores and appropriately aggregate them to form the mixture policy value estimator.

To this end, we assume we have constructed nuisance parameter estimates $\muhati$ and $\ohati$ that satisfy Assumption~\ref{ass:FiniteSampleError}. Then, for each data point $(\Xij,\Aij,\Yij)$ in the observational data set of source $\s\in\sourceSet$, we define the \textit{approximate source AIPW} score for each $a\in\calA$ to be
\begin{equation*}
    \Ghatij(a)\coloneqq\muhati(\Xij;a)+\hat{\widebar{Y}}_i^\s\cdot\ohati(\Xij;a)\cdot\ones\{\Aij=a\}.
\end{equation*}
where
\begin{math}
    \hat{\widebar{Y}}_i^\s=\Yij-\muhati(\Xij;a)
\end{math}
are the approximate centered outcomes. Using these estimated scores, we introduce the following \textit{mixture policy value estimate} and the corresponding \textit{mixture regret estimate}:
\begin{gather*}
    \Qhatw(\pi)\coloneqq\sumijw\Ghatij(\pi(\Xij)) \\
    \Rhatw(\pi)\coloneqq\max_{\pi'\in\Pi}\Qhatw(\pi') - \Qhatw(\pi).
\end{gather*}

Our proposed estimator is a generalized aggregate version of the doubly robust estimator introduced in the standard offline policy learning setting \citep{zhou2023offline}. It is doubly robust in the sense that it is accurate as long as one of the nuisance parameter estimates is accurate for each source. To ensure we can use the same data to estimate the nuisance parameters and to construct the policy value estimates, we utilize a \textit{cross-fitting} strategy for each source. See Appendix~\ref{app:CAIPWEstimation} for more details on the cross-fitting strategy to estimate AIPW scores.

\subsection{Optimization Objective}\label{sec:OptimizationObjective}

Using our mixture regret estimator, we propose an optimization objective to determine a decision policy that robustly performs well under any valid mixture. The optimization objective is to find a policy $\pihat\in\Pi$ that minimizes the worst-case mixture regret estimate over valid mixture weights $\w\in\W$:\begin{equation*}
    \pihat=\argmin_{\pi\in\Pi}\max_{\w\in\W}\Rhatw(\pi)
\end{equation*}

The primary difficulty with this optimization problem is that it is, in general, non-convex-concave so the order of optimization cannot generally be interchanged, and it is non-differentiable so standard gradient-based minimax optimization procedures are not viable. Thus, in the following section, we introduce an algorithm based on no-regret dynamics to solve the stochastic reformulation of this minimax optimization problem.

\section{ALGORITHM}\label{sec:Algorithm}

\subsection{OPO Oracle}

First, we assume that we have access to a standard offline policy optimization (OPO) oracle.

\begin{definition}[OPO Oracle]
    Given a dataset of the form $\{(x_i,\gamma_i(a_1),\dots,\gamma_i(a_d))\}_{i=1}^N\subset\calX\times\calY^d$, an \textit{OPO oracle} for a policy class $\Pi$ solves the following optimization problem:
    \begin{equation*}
        \argmax_{\pi\in\Pi}\sum_{i=1}^N \gamma_i(\pi(x_i)).
    \end{equation*}
\end{definition}

Any procedure that solves the standard offline policy learning problem is a candidate oracle. For example, we can use the PolicyTree method for finite-depth tree policy classes \citep{sverdrup2020policytree}. For linear classes and other parametric policy classes, we can use cost-sensitive classification methods \citep{beygelzimer2008machine, beygelzimer2009error, krishnamurthy2017active}.

\subsection{Exponentiated Gradient OPO}

Assume for now that $\W$ is a finite set. Let $\Delta(\W)$ denote the set of distributions over $\W$, and let $\Delta(\Pi)$ denote the set of distributions over $\Pi$. We can readily rewrite the minimax optimization problem presented in Section~\ref{sec:OptimizationObjective} as a convex-concave optimization over these probability spaces:
\begin{equation*}
    \minPi\max_{\w\in\W}\Rhatw(\pi)=\min_{P\in\Delta(\Pi)}\max_{\rho\,\in\Delta(\W)}\opsE_{\pi\sim P,\,\w\sim\rho}\regbra{\Rhatw(\pi)}.
\end{equation*}
We can approximately solve this zero-sum game through a no-regret dynamics approach \citep{freund1996game}. In particular, note that for a given choice of strategy $\rho\in\Delta(\W)$ by the maximizing player, their corresponding adversarial reward $r(\rho)$ defined by the minimizing player has the following equivalence:
\begin{equation*}
    r(\rho)\coloneqq\min_{P\in\Delta(\Pi)}\opsE_{\pi\sim P,\, \w\sim\rho}\regbra{\Rhatw(\pi)}=\minPi\opsE_{\w\sim\rho}\regbra{\Rhatw(\pi)}.
\end{equation*}
Therefore, we can use the following best response strategy for the minimizing player:
\begin{align*}
    &\argmin_{\pi\in\Pi}\opsE_{\w\sim\rho}\regbra{\Rhatw(\pi)} \\
    &=\argmax_{\pi\in\Pi}\opsE_{\w\sim\rho}\regbra{\Qhatw(\pi)} \\
    &=\argmax_{\pi\in\Pi}\sum_{\w\in\We}\rho_\w\sumiM\frac{\wi}{\ni}\sumjni\Ghatij(\pi(\Xij)).
\end{align*}
The last equality shows that this best response for the minimizing player can be computed through a single call to the OPO oracle on the following dataset: $\{(\Xij,\alpha_\w^\s\Gij(a_1),\dots,\alpha_\w^\s\Gij(a_d))\mid \w\in\W, \source\in\sourceSet, i\in[\ni]\}$ where $\alpha_\w^\s=\rho_\w\cdot\wi/\ni$. Therefore, given we can compute the corresponding adversarial reward for the maximizing player, we can use the exponentiated gradient (EG) algorithm \citep{kivinen1997exponentiated} to sequentially learn a maximizing distribution over a finite collection of experts with each $\w\in\W$ being an expert. More formally, we initialize the expert distribution to $\rho^{(1)}\,\propto\,1$ and sequentially update:
\begin{gather*}
    \pi^{(t)}=\argmax_{\pi\in\Pi}\opsE_{\w\sim\rho^{(t)}}\regbra{\Qhatw(\pi)} \\
    \rho_\w^{(t+1)}\,\propto\,\rho_\w^{(t)}\cdot\exp\bigrbra{\eta\cdot g_\w^{(t)}},
\end{gather*}
where $g_\w^{(t)}$ is the $\w$-th element of the gradient of the observed adversarial reward $r^{(t)}(\rho^{(t)})$, i.e.,
\begin{align*}
    g_\w^{(t)}&\!\coloneqq\!\frac{\partial}{\partial\rho_\w^{(t)}}r^{(t)}(\rho^{(t)})\!=\!\Rhatw(\pi^{(t)})\!=\!\max_{\pi'\in\Pi}\Qhatw(\pi')-\Qhatw(\pi^{(t)})
\end{align*}
and $\eta=\sqrt{\frac{\log|\W|}{\bar{B}^2T}}$ with a constant $\bar{B}$ that uniformly bounds the empirical mixture regret (see Corollary \ref{cor:HedgeOPOSuboptimalityBound} in Appendix~\ref{app:EGOPO-Bound} for a suitable high probability bound based on $B$). We run this repeated procedure for $T$ steps and return the policy $\pihat=\pi^{(T)}$. Note that the gradient $g_\w^{(t)}$ can also be computed with a call to the OPO oracle for each $\w\in\W$ to determine $\max_{\pi'\in\Pi}\Qhatw(\pi')$ on the following dataset $\{(\Xij,\alpha_\w^\s\Gij(a_1),\dots,\alpha_\w^\s\Gij(a_d))\mid \source\in\sourceSet, i\in[\ni]\}$ where $\alpha_\w^\source=\wi/\ni$.

In the case where $\W$ is not finite, we can simply perform the same procedure over a \textit{minimal covering set} of $\W$. In our regret bound, we account for any incurred approximation error from this discretization. Additionally, we will account for the approximation error for terminating the algorithm after finitely many steps. See Algorithm \ref{alg:HedgeOPO} for a pseudocode outline of the algorithm, referred to as EG-OPO, and Appendix~\ref{app:AdditionalAlgorithmDetails} for more details on the algorithm complexity, the choice of OPO oracles, and nuisance parameter estimation.

\begin{figure}
        \begin{algorithm}[H]
        \caption{EG-OPO}
        \label{alg:HedgeOPO}
        \begin{algorithmic}[1]
        
        \REQUIRE OPO oracle, time horizon $T$, approximation error threshold $\epsilon$, uniform regret bound $\bar{B}$

        \STATE Compute minimal $\epsilon$-covering $\We$ of $\W$ under the $\ell_1$ distance

        \STATE Set $\eta=\sqrt{\frac{\log|\We|}{\bar{B}^2 T}}$

        \STATE Set $\rho_\w^{(1)}\ \propto\ 1$ for all $\w\in\We$
        
        \FOR{each round $t=1,2,\dots,T$}
            \STATE Utilize the OPO oracle to solve for the minimizing player's best response to $\rho^{(t)}$: $$\pi^{(t)}=\argmax_{\pi\in\Pi}\opsE_{\w\sim\rho^{(t)}}\regbra{\Qhatw(\pi)}$$

            \STATE Utilize OPO oracle to compute gradient of the adversarial reward for the maximizing player:
            $$g_\w^{(t)}=\Rhatw(\pi^{(t)}),\ \forall\w\in\We$$

            \STATE Update the maximizing player's distribution with the exponentiated gradient ascent update:
            $$\rho_\w^{(t+1)}\ \propto\ \rho_\w^{(t)}\cdot\exp(\eta \cdot g_\w^{(t)}),\ \forall\w\in\We$$
        \ENDFOR
        \RETURN $\pihat=\pi^{(T)}$
        \end{algorithmic}
        \end{algorithm}
\end{figure}

\section{REGRET BOUNDS}\label{sec:RegretBounds}

In this section, we establish regret bounds for the policy solution to our optimization objective in Section~\ref{sec:OptimizationObjective}. Refer to Appendices \ref{app:AuxiliaryResults}, \ref{app:PolicyClassComplexityMeasures}, \ref{app:BoundingGlobalRegret} for detailed discussions and proofs of the results presented in this section.

\subsection{Complexity and Skewness}\label{sec:ComplexityAndSkewness}

First, we introduce important quantities that appear in our regret bounds.

\vspace{-.5em}
\paragraph{Policy Class Complexity}
The following quantity provides a measure of policy class complexity based on a variation of the classical entropy integral introduced by \citet{dudley1967sizes}, and it is useful in establishing a class-dependent regret bound. See Appendix~\ref{app:HammingDistance} for more details on its definition.

\begin{definition}[Entropy integral]
Let
\begin{math}
    \Ham(\pi_1,\pi_2;x)\coloneqq\frac{1}{m}\sum_{i=1}^{m}\ones\{\pi_1(x_i)\neq\pi_2(x_i)\}
\end{math}
be the Hamming distance between any two policies $\pi_1,\pi_2\in\Pi$ given a covariate set $x\subset\calX$ of size $m\in\N$.
The \textit{entropy integral} of a policy class $\Pi$ is
\begin{equation*}
    \kappa(\Pi)\coloneqq\int_0^1\sqrt{\log N_{\Ham}(\epsilon^2,\Pi)}d\epsilon,
\end{equation*}
where $N_{\Ham}(\epsilon^2,\Pi)$ is the maximal $\epsilon^2$-covering number of $\Pi$ under the Hamming distance over covariate sets of arbitrary size.
\end{definition}

The entropy integral is constant for a fixed policy class, and rather weak assumptions on the class are sufficient to ensure it is finite such as sub-exponential growth on its Hamming covering number, which is satisfied by many policy classes including parametric and finite-depth tree policy classes \citep{zhou2023offline}. In the binary action setting, the entropy integral of a policy class relates to its VC-dimension with $\kappa(\Pi)=\sqrt{\smash[b]{\VC(\Pi)}}$, and for $D$-dimensional linear classes $\kappa(\Pi)=\calO(\sqrt{\smash[b]{D}}\,)$.

\vspace{-.5em}
\paragraph{Source Skewness}
The following quantity measures the imbalance of the source sampling distribution $\w$ relative to the \textit{empirical distribution of samples across sources}. This quantity naturally arises in the generalization bounds of weighted mixture distributions \citep{mohri2019agnostic,mansour2021theory}.

\begin{definition}[Skewness]
    The \textit{skewness} of a given set of mixture weights $\w$ relative to the empirical distribution of samples $\barn\coloneqq(\ni/n)_{\source\in\sourceSet}$ is
    \begin{equation*}
        \skewness\coloneqq 1+\chi^2(\w||\barn),
    \end{equation*}
    where $\chi^2(\w||\barn)$ is the chi-squared divergence of $\w$ from $\barn$. Additionally, the \textit{mixture-agnostic skewness} is
    \begin{equation*}
        \skewnessW\coloneqq\maxwW\skewness.
    \end{equation*}
\end{definition}

\subsection{Mixture-Agnostic Regret Bound}

The following result captures a root-$n$ finite-sample bound for the mixture-agnostic regret that parallels the optimal regret bounds typically seen in the offline policy learning literature.

\begin{restatable}[Mixture-Agnostic Regret Bound]{theorem}{MainTheorem}\label{thm:MainTheorem}
    Suppose Assumptions~\ref{ass:dgp}, \ref{ass:LocalDataSizeScaling}, and \ref{ass:FiniteSampleError} hold.
    For any $\epsilon>0$, let $\W_\epsilon$ denote a minimal $\epsilon$-covering set of $\W$ under the $\ell_1$ distance.
    Set $T=(n/\skewnessW)^{1+\alpha}$ for any choice of $\alpha>0$. Then, for any $\epsilon>0$ and any $\delta\in(0,1)$,
    with probability at least $1-\delta$, the EG-OPO policy $\pihat=\pi^{(T)}$ achieves the regret bound
    \begin{align*}
        \Rwp(\pihat)\le \minPi\maxwW\Rw(\pi) + 2B\epsilon + \rateW
    \end{align*}
    for any $\wp\in\W$, where
    \begin{equation*}
        \rateW\!\coloneqq\! C_{\epsilon,\delta}\kappa(\Pi)\sqrt{V\cdot\fracskewnessWn} +\littleop{\!\sqrtfracskewnessWn},
    \end{equation*}
    where the constant $B=\maxC B_s$ is a uniform upper bound on all the potential outcomes across sources, $C_{\epsilon,\delta}=\sqrt{C\log(\regabs{\We}/\delta)}$ with universal constant $C$, and $V=\maxC\maxPi\E_{\calDcec}\!\regbra{\Gi(\pi(\Xi))^2}$ is the worst-case AIPW score variance across sources.
\end{restatable}

This result shows that the EG-OPO algorithm achieves minimal worst-case mixture regret uniformly across any valid mixture distribution, with a vanishing rate proportional to $\calO_p(\kappa(\Pi)\sqrt{V\cdot\skewnessW/n})$. This ensures the policy's robustness across a broad range of target settings. Importantly, the trained policy retains strong performance on each individual source distribution while generalizing to new settings represented by a mixture of source distributions, broadening the applicability of policy learning.

Note that if we set the discretization granularity to be $\epsilon=o(\sqrt{\skewnessW/n})$, then the $2B\epsilon$ term in the above bound can be absorbed into the vanishing error term $o_p(\sqrt{\skewnessW/n})$.
However, depending on the corresponding size of $\regabs{\We}$ and other terms
in the above bound, more favorable choices may be possible.

Additionally, the root-$n$ vanishing rate is moderated by the skewness which can also scale with the total sample size. For example, if $\W=\{\barn\}$ then $\skewnessW/n=1/n$, and if $\W=\{(1,0,\dots,0)\}$ then $\skewnessW/n=1/n_1$. Thus, this skewness-moderated rate generalizes and smoothly interpolates between the rates one expects from the uniform weighted model and the single source model. Indeed, when sources are identical and $\W=\{\barn\}$, we effectively recover the best known rates from standard offline policy learning \citep{zhou2023offline}.

\subsection{Proof Sketch}

In this section, we provide a high-level sketch of the proof of Theorem~\ref{thm:MainTheorem}. Refer to the appendix for the full proofs and discussions.

First, we show the mixture regret $\Rwp(\pihat)$ can be decomposed into the \textit{empirical mixture regret} $\Rhatwp(\pihat)$ and the \textit{mixture empirical process}
$\sup_{\pia,\pib}\regabs{\Deltaw(\pia,\pib)-\Deltahatw(\pia,\pib)},$
where $\Deltaw(\pia,\pib)=\Qw(\pia)-\Qw(\pib)$ and $\Deltahatw(\pia,\pib)=\Qhatw(\pia)-\Qhatw(\pib)$ are the \textit{true} and \textit{empirical} \textit{mixture policy value differences} between any two policies $\pia,\pib\in\Pi$.

The empirical mixture regret can be bounded by the suboptimality bound of the EG-OPO algorithm using the following lemma which we demonstrate follows from the results of \citep{freund1996game}.
\begin{restatable}{lemma}{SuboptimalityBoundLemma}\label{lem:EG-OPO-SuboptimalityBound}
    For any $T$ and any $\wp\in\We$,
    \begin{equation*}
        \Rhatwp(\pihat)\le\minPi\maxwW\Rhatw(\pi)+2\hat{B}\sqrt{\frac{\log|\We|}{T}}
    \end{equation*}
    where $\hat{B}$ is a uniform bound on $\Rhatwp(\pi)$.
\end{restatable}

Next, the mixture empirical process can be further decomposed into \textit{oracle} and \textit{approximate} \textit{mixture empirical processes}. For the oracle empirical process, we make use of Talagrand concentration inequalities and symmetrization arguments \citep{koltchinskii2011oracle} to establish tight high-probability bounds based on the \textit{weighted Rademacher complexity} \citep{mohri2019agnostic} for the policy value function class. We then follow a Dudley chaining argument \citep{dudley1967sizes} to show that this weighted Rademacher complexity can be bounded by measures of policy class complexity, source mixture set skewness, and worst-case source AIPW score variance as they appear in Theorem~\ref{thm:MainTheorem}.

The approximate empirical process can simply be bounded directly by decomposing the differences and relating them to the assumptions on the data-generating process and the nuisance parameter estimation error stated in Assumptions~\ref{ass:dgp}, \ref{ass:LocalDataSizeScaling}, \ref{ass:FiniteSampleError}. For robustness across valid mixtures, for each of these empirical process bounds we take a union bound over the valid mixture weights set.

Moreover, the resulting bound on the mixture empirical process allows us to relate the minimax empirical mixture regret that appears in Lemma~\ref{lem:EG-OPO-SuboptimalityBound} to the minimax mixture regret that finally appears in Theorem~\ref{thm:MainTheorem}. Lastly, the additional $\epsilon$ term in the mixture-agnostic regret bound is due to the discretization error over the minimal cover $\We$ of the valid mixture weights set $\W$.

\subsection{Target Regret Bound}

The following result captures a regret bound on a target distribution that is not captured by the family of valid mixture distributions. Refer to Appendix~\ref{app:BoundingTargetRegret} for a proof of this result.

\begin{restatable}[Target Regret Bound]{theorem}{TargetTheorem}\label{thm:TargetTheorem}
    Let $R(\pihat)$ denote the target regret of $\pihat$ under a target distribution $\calD$ over $\calX\times\calY^d$.
    This is bounded as
    \begin{equation*}
        R(\pihat)\le 2\cdot\disc(\calD_\W, \calD) + \maxwW\Rw(\pihat),
    \end{equation*}
    where
    \begin{equation*}
        \disc(\calD_\W,\calD)=\minwW\maxPi\abs{\Qw(\pi)-Q(\pi)}
    \end{equation*}
    is the minimax mixture policy value discrepancy. Moreover, if the same boundedness condition also holds for $\calD$, then this discrepancy can be further bounded by
    \begin{equation*}
        \disc(\calD_\W,\calD)\le B\cdot\minwW\TV(\calDw,\calD),
    \end{equation*}
    where $\TV$ is the total variation distance.
\end{restatable}

This result immediately follows from a decomposition of the target regret into the mixture regret and its associated discrepancy. Note that the worst-case mixture regret here can be bounded using the result of Theorem~\ref{thm:MainTheorem}.

\section{EXPERIMENTS}\label{sec:Experiments}

\subsection{Setup}
In this section, we describe our experimental setup. We consider the source set $\sourceSet=[k]$ with $k=3$ sources, a binary action set $\calA=\{a_1,a_2\}$, and the context space $\calX=[-1, 1]^p$ with $p=2\times q$ where $q=4$. For every data source $\s\in\sourceSet$, we construct the following data-generating process:
\begin{compactitem}
    \item $\Ai\sim\Uniform(\calA)$,
    \item $\Xi\sim\Uniform(\calX)$,
    \item $\Yi(a)|\Xi\sim\Normal\bigrbra{\mui(\Xi;a),\sigma_\s^2}$ for all $a\in\calA$,
\end{compactitem}
where the source reward mean is $\mui(x;a)=x_a^T\theta_\s$ given source parameter $\theta_\s\sim\Normal(\vec{0}, \sigma^2 I_{q})$ with $\sigma^2=5$ and the source reward variance is $\sigma_\s^2=1$. For a given total sample size $n\in\N$, each source $\s\in\sourceSet$ is allocated a local sample sample size determined by the following increasing function $\nc=\nuc(n)=n/k$. By construction, this entire data-generating process satisfies our data assumptions stated in Section~\ref{sec:Preliminaries-Data}.

For the policy class $\Pi$, we choose the class of fixed depth-2 decision trees, and for the OPO oracle, we use the PolicyTree algorithm \citep{sverdrup2020policytree}. Refer to Appendix~\ref{app:AdditionalAlgorithmDetails} for additional details on the algorithm implementation, namely on the cross-fitting strategy for nuisance parameter and AIPW score estimation.

To assess the robustness of our approach, we will compare the empirical regret of different policies under two settings: a fixed source distribution (source $s=1$) and a nearby mixture distribution (mixture weights $\w=[0.9, 0.05, 0.05]$). Under both settings, we will compare the \textit{EG-OPO policy} against two baseline policies: the \textit{aggregate policy} trained on data aggregated from all sources and the \textit{source policy} trained on an equivalent amount of source data, each trained with a single call to the PolicyTree OPO oracle. The analysis considers a range of training sample sizes up to 500. Our plots will display the rolling mean and standard deviation bands of empirical regrets over three seed runs, with regrets computed relative to corresponding optimal models trained on 2,000 similarly sampled data points.

\subsection{Performance on Source Distribution}

We begin by comparing our algorithm against our baselines on the source distribution for a fixed source $s=1$. In Figure~\ref{fig:source_regret}, we plot the source regret $R_1(\pi)$ for the aggregate policy $\pihat_\calS$, source policy $\pihat_1$, and EG-OPO policy $\pihat_{EG}$. We see that the aggregate policy performs poorly due to distribution shift in the data from all other sources. For a similar reason, the EG-OPO policy does not perform as well as the source policy, but it still shows improvement as the sample size increases, consistent with our theoretical results.

\begin{figure}[ht]
    \centering
    \includegraphics[trim=5pt 0pt 50pt 23pt, clip, width=.75\linewidth]{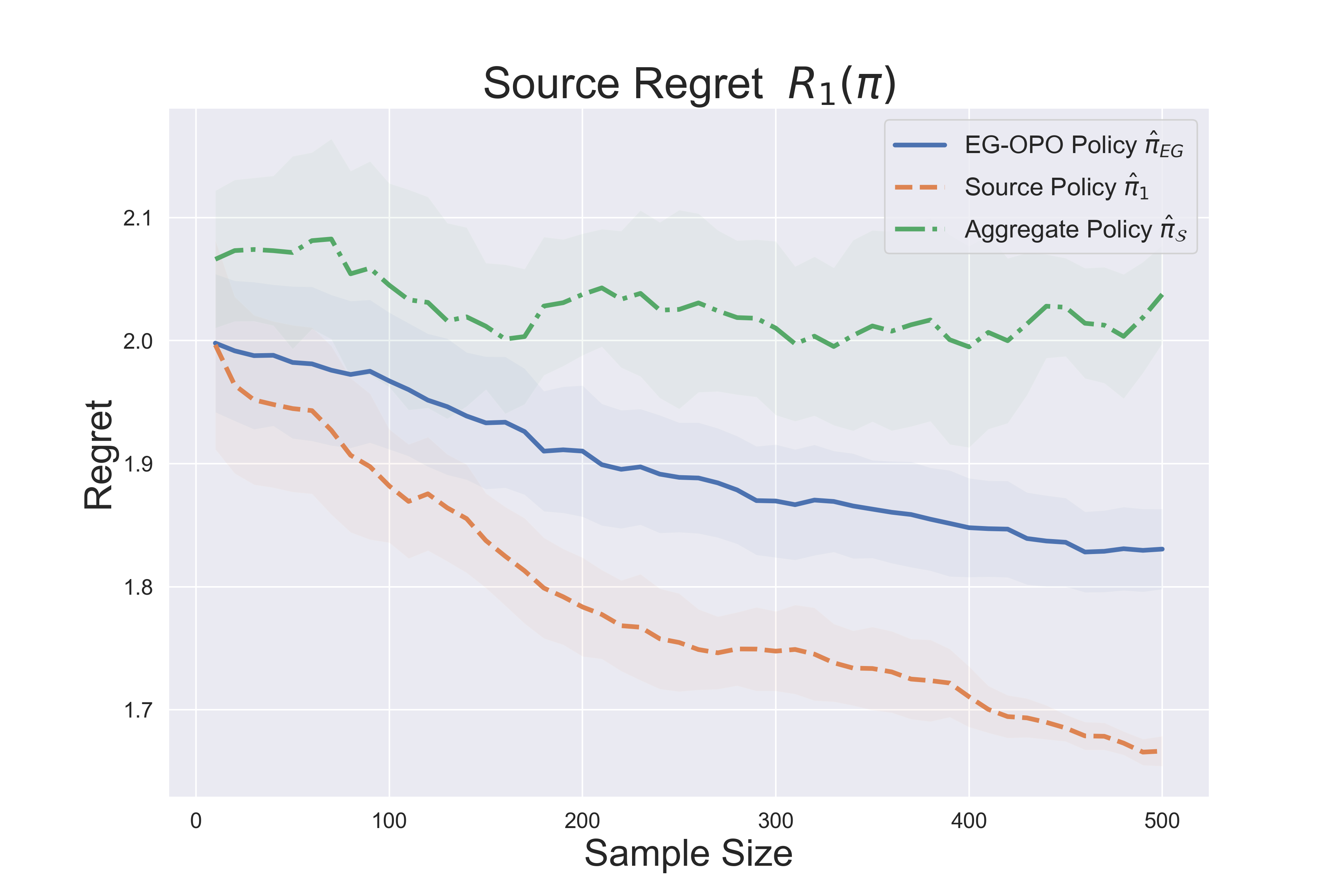}
    \vspace{-1em}
    \caption{Empirical source regret $R_1(\pi)$.}
    \label{fig:source_regret}
    \vspace{-1em}
\end{figure}

\subsection{Performance on Mixture Distribution}

Next, we compare performances on a mixture distribution close to the source distribution $\calD_1$, namely $\calDw$ where $\w=[0.9, 0.05, 0.05]$. In Figure~\ref{fig:mixture_regret}, we plot the mixture regret $R_\lambda(\pi)$ for the same policies. The aggregate policy again performs poorly, as the mixture distribution, although it accounts for other sources, is not close to the uniform distribution. Similarly, the local policy struggles due to the slight distribution shift, despite the mixture weights being close to the source degenerate weights, and performs even worse than the aggregate policy. In contrast, the EG-OPO policy remains robust and performs well in this setting, further supporting the consistent performance of our approach across different distributions uniformly at once.

\begin{figure}[ht]
    \centering
    \includegraphics[trim=5pt 0pt 50pt 23pt, clip, width=.75\linewidth]{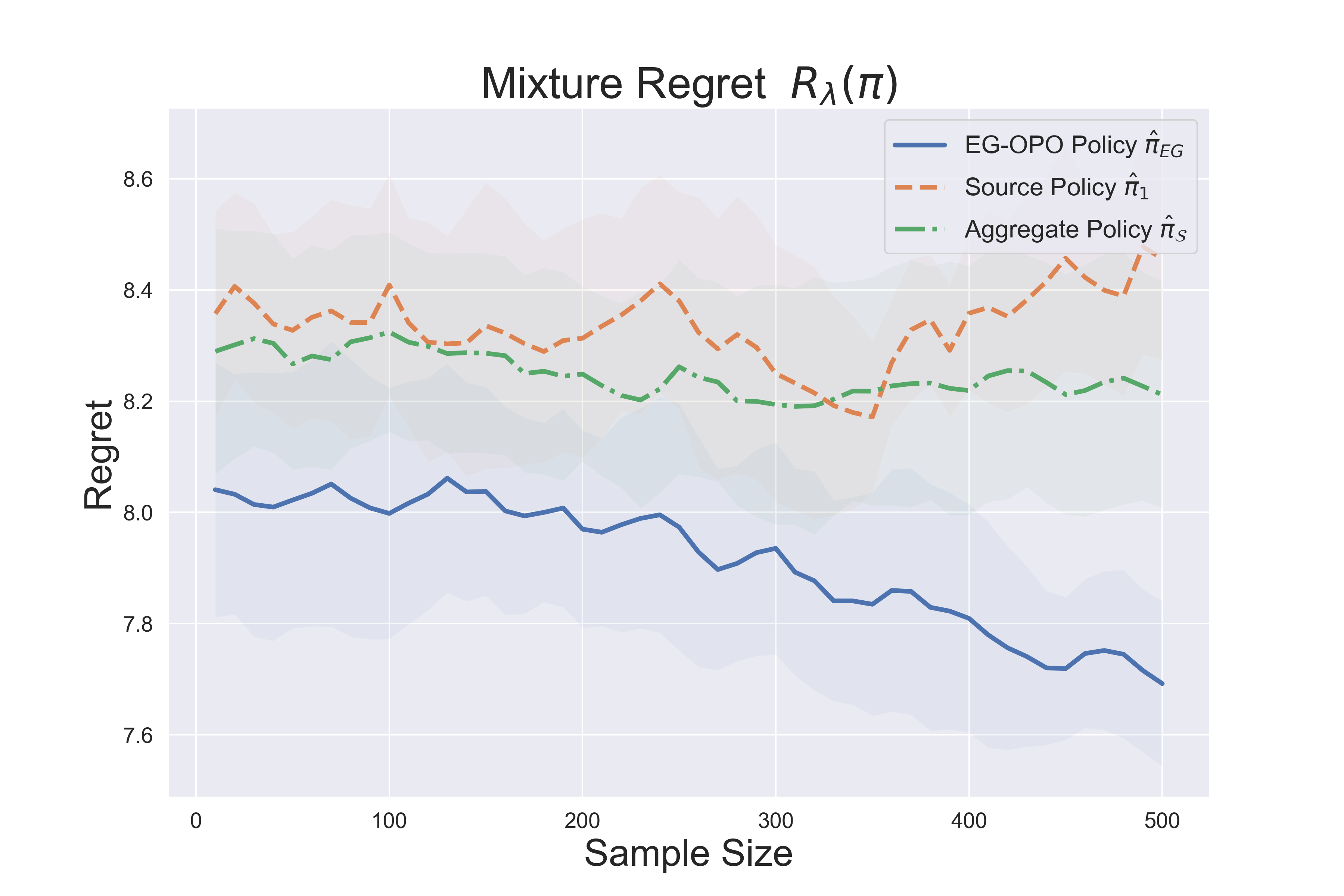}
    \vspace{-1em}
    \caption{Empirical mixture regret $\Rw(\pi)$.}
    \label{fig:mixture_regret}
    \vspace{-1em}
\end{figure}

\section{CONCLUSION}

We presented an approach for robust offline policy learning using observational data from multiple sources, formulated as minimax regret optimization to ensure low regret across target distributions modeled as source mixtures. Our algorithm, combining doubly robust evaluation and no-regret learning strategies, achieves minimal worst-case regret bounds up to a root-$n$ vanishing rate of the total data. Our theoretical analysis and experiments confirm its effectiveness in learning robust generalizable policies across different environments.

%%%%%%%%%%%%%%%%%%%%
% References
%%%%%%%%%%%%%%%%%%%%
\bibliography{references}
\bibliographystyle{plainnat}

%%%%%%%%%%%%%%%%%%%%%%%%%%%%%%
% APPENDIX
%%%%%%%%%%%%%%%%%%%%%%%%%%%%%%
\appendix
\onecolumn
\input{supplement}
%%%%%%%%%%%%%%%%%%%%%%%%%%%%%%
%%%%%%%%%%%%%%%%%%%%%%%%%%%%%%

\end{document}

%% file: macros.tex
\usepackage{amsfonts, amsmath, amsthm, amssymb}
\usepackage{thmtools, mathtools,mathabx}
\usepackage{algorithm}
\usepackage{algorithmic}
\usepackage{subcaption}
\usepackage{paralist}
\usepackage{hyperref}
\usepackage{cleveref}
\newcommand\numberthis{\addtocounter{equation}{1}\tag{\theequation}}

\theoremstyle{plain}

\newtheorem{corollary}{Corollary}
\newtheorem{lemma}{Lemma}
\newtheorem{proposition}{Proposition}
\theoremstyle{definition}
\newtheorem{definition}{Definition}
\newtheorem{assumption}{Assumption}
\theoremstyle{remark}
\newtheorem*{remark}{Remark}

\DeclareMathOperator*{\argmax}{arg\,max}
\DeclareMathOperator*{\argmin}{arg\,min}
\DeclareMathOperator*{\opsE}{\mathbb{E}}
\DeclareMathOperator{\E}{\mathbb{E}}
\let\P\relax
\DeclareMathOperator{\P}{\mathbb{P}}

\DeclareMathOperator{\TV}{TV}

\DeclareMathOperator{\Normal}{Normal}
\DeclareMathOperator{\Uniform}{Uniform}
\DeclareMathOperator{\disc}{disc}

\newcommand{\eps}{\varepsilon}
\newcommand{\ones}{\mathbf{1}}

\newcommand{\indep}{\perp \!\!\! \perp}

\newcommand{\supF}{\sup_{f\in\calF}}
\newcommand{\supH}{\sup_{h\in\calH}}

\newcommand{\abs}[1]{\left|#1\right|}
\newcommand{\bra}[1]{\left[#1\right]}
\newcommand{\rbra}[1]{\left(#1\right)}

\newcommand{\norm}[1]{\left\lVert#1\right\rVert}

\newcommand{\regabs}[1]{|#1|}
\newcommand{\regbra}[1]{[#1]}
\newcommand{\regnorm}[1]{\lVert#1\rVert}

\newcommand{\bigabs}[1]{\big|#1\big|}
\newcommand{\bigbra}[1]{\big[#1\big]}
\newcommand{\bigrbra}[1]{\big(#1\big)}

\newcommand{\bignorm}[1]{\big\lVert#1\big\rVert}

\newcommand{\Bigabs}[1]{\Big|#1\Big|}
\newcommand{\Bigbra}[1]{\Big[#1\Big]}
\newcommand{\Bigrbra}[1]{\Big(#1\Big)}
\newcommand{\Bigmid}{\ \Big|\ }

\newcommand{\biggabs}[1]{\bigg|#1\bigg|}
\newcommand{\biggbra}[1]{\bigg[#1\bigg]}
\newcommand{\biggrbra}[1]{\bigg(#1\bigg)}

\newcommand{\Biggabs}[1]{\Bigg|#1\Bigg|}

\newcommand{\bigO}[1]{\calO\left(#1\,\right)}
\newcommand{\littleo}[1]{o\left(#1\,\right)}
\newcommand{\bigOp}[1]{\calO_p\left(#1\,\right)}
\newcommand{\littleop}[1]{o_p\left(#1\,\right)}

\newcommand{\N}{\mathbb{N}}
\newcommand{\R}{\mathbb{R}}
\newcommand{\calA}{\mathcal{A}}

\newcommand{\calD}{\mathcal{D}}

\newcommand{\calF}{\mathcal{F}}

\newcommand{\calH}{\mathcal{H}}

\newcommand{\calO}{\mathcal{O}}

\newcommand{\calS}{\mathcal{S}}

\newcommand{\calU}{\mathcal{U}}

\newcommand{\calX}{\mathcal{X}}
\newcommand{\calY}{\mathcal{Y}}
\newcommand{\calZ}{\mathcal{Z}}

\newcommand{\barn}{\bar{n}}
\newcommand{\tilden}{\tilde{n}}
\newcommand{\tildex}{\tilde{x}}
\newcommand{\tildez}{\tilde{z}}

\renewcommand{\o}{\omega}

\newcommand{\barV}{\bar{V}}
\newcommand{\frakR}{\mathfrak{R}}
\newcommand{\Ham}{\text{H}}

\newcommand{\VC}{\text{VC}}

\newcommand{\Bhat}{\hat{B}}
\newcommand{\Qtilde}{\tilde Q}

\newcommand{\ohat}{\hat{\o}}
\newcommand{\alphaK}{\alpha_K}
\newcommand{\zetamu}{{\zeta_\mu}}

\newcommand{\zetao}{{\zeta_w}}
\newcommand{\sumin}{\sum_{i=1}^{n}}
\newcommand{\sumaA}{\sum_{a\in\calA}}
\newcommand{\sumkK}{\sum_{k=1}^K}

\newcommand{\pia}{\pi_a}
\newcommand{\pib}{\pi_b}
\newcommand{\pistar}{\pi^*}
\newcommand{\pihat}{\hat{\pi}}

\newcommand{\minPi}{\min_{\pi\in\Pi}}
\newcommand{\maxPi}{\max_{\pi\in\Pi}}
\newcommand{\supPi}{\sup_{\pi\in\Pi}}
\newcommand{\supPiab}{\sup_{\pi_a,\pi_b\in\Pi}}

\newcommand{\rateW}{\xi_{\epsilon,\delta}(n;\Pi,\W)}
\newcommand{\sourceSet}{\calS}
\newcommand{\nC}{n_{\sourceSet}}
\newcommand{\weight}{\lambda}
\newcommand{\w}{\weight}

\renewcommand{\wp}{{\weight'}}
\newcommand{\W}{\Lambda}
\newcommand{\We}{{\Lambda_\epsilon}}

\newcommand{\minwW}{\min_{\w\in\W}}
\newcommand{\maxwW}{\max_{\w\in\W}}

\newcommand{\Qw}{Q_\w}

\newcommand{\Qwp}{Q_\wp}
\newcommand{\Qhatw}{\hat{Q}_\w}

\newcommand{\Qhatwp}{\hat{Q}_\wp}
\newcommand{\Qtildew}{\tilde{Q}_\w}
\newcommand{\Rw}{R_\w}

\newcommand{\Rwp}{R_{\wp}}

\newcommand{\Rhatw}{\hat{R}_\w}

\newcommand{\Rhatwp}{\hat{R}_\wp}

\newcommand{\pihatwp}{\hat{\pi}_\wp}
\newcommand{\pistarw}{\pi^*_\w}

\newcommand{\pistarwp}{\pi^*_\wp}

\newcommand{\ellw}{\ell_{\w}}
\newcommand{\Deltaw}{\Delta_\w}

\newcommand{\Deltawp}{\Delta_\wp}
\newcommand{\Deltahatw}{\hat{\Delta}_\w}

\newcommand{\Deltahatwp}{\hat{\Delta}_\wp}
\newcommand{\Deltatildew}{\tilde{\Delta}_\w}

\newcommand{\swnC}{s_{\w,\nC}}
\newcommand{\UwnC}{U_{\w,\nC}}
\newcommand{\skewness}{\mathfrak{s}(\w\|\barn)}
\newcommand{\fracskewnessn}{\frac{\skewness}{n}}
\newcommand{\sqrtfracskewnessn}{\sqrt{\fracskewnessn}}
\newcommand{\skewnessW}{\mathfrak{s}(\W\|\barn)}
\newcommand{\fracskewnessWn}{\frac{\skewnessW}{n}}
\newcommand{\sqrtfracskewnessWn}{\sqrt{\fracskewnessWn}}

\newcommand{\frakRwn}{\mathfrak{R}_{\w,\nC}}
\newcommand{\frakRwnF}{\frakR_{\w,\nC}(\calF_\Pi)}

\newcommand{\Bw}{B_{\w,\nC}}
\newcommand{\BwW}{B_{\w,\nC}(Z)}

\newcommand{\Vwn}{V_{\w,\nC}}
\newcommand{\Uwn}{U_{\w,\barn}}
\newcommand{\fracVwn}{\frac{\Vwn}{n}}
\newcommand{\sqrtfracVwn}{\sqrt{\fracVwn}}

\newcommand{\calDw}{\calD_\w}

\newcommand{\source}{s}
\newcommand{\s}{\source}

\newcommand{\nc}{n_{\s}}

\newcommand{\wi}{\weight_\s}
\newcommand{\barni}{\barn_\s}

\newcommand{\maxC}{\max_{\s\in\sourceSet}}

\newcommand{\sumiM}{\sum_{\s\in\sourceSet}}

\newcommand{\mui}{\mu_\s}
\newcommand{\ei}{e_\s}
\newcommand{\oi}{\o_\s}

\newcommand{\Qi}{Q_\s}
\renewcommand{\ni}{n_\s}
\renewcommand{\Xi}{X^\s}
\newcommand{\Ai}{A^\s}

\newcommand{\Yi}{Y^\s}
\newcommand{\Zi}{Z^\s}

\newcommand{\vecYi}{\vec{Y}^\s}
\newcommand{\Gi}{\Gamma^\s}

\newcommand{\Xij}{X_{i}^{\s}}
\newcommand{\Aij}{A_{i}^{\s}}

\newcommand{\Yij}{Y_{i}^{\s}}
\newcommand{\Zij}{Z_{i}^{\s}}
\newcommand{\Zijprime}{Z_{i}^{\prime\s}}

\newcommand{\Zhatij}{\hat{Z}_{i}^{\s}}
\newcommand{\Gij}{\Gamma_{i}^{\s}}
\newcommand{\Gijp}{\Gamma_{i}^{\s\prime}}
\newcommand{\Gijpp}{\Gamma_{i}^{\s\prime\prime}}
\newcommand{\Gijppp}{\Gamma_{i}^{\s\prime\prime\prime}}
\newcommand{\Ghatij}{\hat{\Gamma}_{i}^{\s}}

\newcommand{\vecGij}{\vec{\Gamma}_{i}^{\s}}
\newcommand{\epsij}{\varepsilon_{i}^{\s}}

\newcommand{\kij}{k_\s(i)}

\newcommand{\muhatikij}{\hat{\mu}_\s^{-k_\s(i)}}

\newcommand{\ohatikij}{\ohat_\s^{-k_\s(i)}}

\newcommand{\Di}{\calD_\s}
\newcommand{\calDc}{\calD_\s}

\newcommand{\tildecalDc}{\tilde{\calD}_\s}
\newcommand{\calDcec}{\bar{\calD}_\s}

\newcommand{\gij}{\gamma_{i}^{\s}}

\newcommand{\xij}{x_{i}^{\s}}
\newcommand{\zij}{z_{i}^{\s}}

\newcommand{\x}{\{\xij\mid \s\in\sourceSet, i\in[\nc]\}}

\newcommand{\nuc}{\nu_{\s}}
\newcommand{\gc}{g_\s}

\newcommand{\tildenij}{\tilde{n}_{i}^{\s}}
\newcommand{\tildexij}{\tilde{x}_{i}^{\s}}
\newcommand{\tildeZij}{\tilde{Z}_{i}^{\s}}

\newcommand{\muhati}{\hat{\mu}_\s}

\newcommand{\ohati}{\hat{\o}_\s}

\newcommand{\Qhati}{\hat{Q}_\s}

\newcommand{\fracwini}{\frac{\w_\s}{\nc}}

\newcommand{\fracwisqni}{\frac{\w_\s^2}{\nc}}
\newcommand{\fracwinisq}{\frac{\w_\s^2}{\nc^2}}

\newcommand{\fracwibarni}{\frac{\w_\s}{\barn_\s}}
\newcommand{\fracwisqbarni}{\frac{\w_\s^2}{\barn_\s}}

\newcommand{\sumjni}{\sum_{i=1}^{\nc}}
\newcommand{\sumijw}{\sumiM\fracwini\sumjni}
\newcommand{\sumijwsq}{\sumiM\sumjni\fracwinisq}
\newcommand{\sumij}{\sumiM\sumjni}

%% file: supplement.tex
\section{AUXILIARY RESULTS}\label{app:AuxiliaryResults}

The following known results will be used in our regret bound proofs. See Chapter 2 of \cite{koltchinskii2011oracle} for discussions of these results.

\begin{lemma}[Hoeffding's inequality]\label{lem:HoeffdingInequality}
    Let $Z_1,\dots,Z_n$ be independent random variables with $Z_i\in[a_i,b_i]$ almost surely. For all $t>0$, the following inequality holds
    \begin{equation*}
        \P\rbra{\biggabs{\sum_{i=1}^nZ_i-\E\bra{Z_i}}\ge t}\le 2\exp\rbra{-\frac{2t^2}{\sum_{i=1}^n(b_i-a_i)^2}}.
    \end{equation*}
\end{lemma}

\begin{lemma}[Talagrand's inequality]\label{lem:TalagrandInequality}
    Let $Z_1,\dots,Z_n$ be independent random variables in $\calZ$. For any class of real-valued functions $\calH$ on $\calZ$ that is uniformly bounded by a constant $U>0$ and for all $t>0$, the following inequality holds
    \begin{equation*}
        \P\rbra{\biggabs{\supH\Bigabs{\sum_{i=1}^n h(Z_i)} - \E\Bigbra{\supH\Bigabs{\sum_{i=1}^n h(Z_i)}}}\ge t}\le C\exp\rbra{-\frac{t}{CU}\log\rbra{1+\frac{Ut}{D}}},
    \end{equation*}
    where $C>0$ is a universal constant and $D\ge\E\bra{\supH\sum_{i=1}^n h^2(Z_i)}$.
\end{lemma}

\begin{lemma}[Ledoux-Talagrand contraction inequality]\label{lem:ContractionInequality}
    Let $Z_1,\dots,Z_n$ be independent random variables in $\calZ$. For any class of real-valued functions $\calH$ on $\calZ$ and any $L$-Lipschitz function $\varphi$, the following inequality holds
    \begin{equation*}
        \E\bra{\supH\biggabs{\sum_{i=1}^n\eps_i(\varphi\circ h)(Z_i)}}\le 2L\E\bra{\supH\biggabs{\sum_{i=1}^n\eps_i h(Z_i)}},
    \end{equation*}
    where $\eps_1,\dots,\eps_n$ are independent Rademacher random variables.
\end{lemma}

Lastly, we state an auxiliary inequality that serves as a typical candidate for the quantity denoted by $D$ above in Talagrand's inequality. This result follows as a corollary of the Ledoux-Talagrand contraction inequality and a symmetrization argument. We provide a proof for completeness.

\begin{lemma}\label{lem:ExpectedSupSumBound}
    Let $Z_1,\dots,Z_n$ be independent random variables in $\calZ$. For any class of real-valued functions $\calH$ on $\calZ$ and any $L$-Lipschitz function $\varphi$, the following inequality holds
    \begin{equation*}
        \E\bra{\supH\sum_{i=1}^n (\varphi\circ h)(Z_i)}\le\supH\sum_{i=1}^n\E\bigbra{(\varphi\circ h)(Z_i)}+4L\E\bra{\supH\biggabs{\sum_{i=1}^n\eps_ih(Z_i)}},
    \end{equation*}
    where $\eps_1,\dots,\eps_n$ are independent Rademacher random variables.
\end{lemma}
\begin{proof}
    We have that
    \begin{align}
        &\E\bra{\supH\sumin(\varphi\circ h)(Z_i)}-\supH\sumin\E\bigbra{(\varphi\circ h)(Z_i)} \label{eq:Lemma-ExpectedSupSumBound-eq1} \\
        &=\E\bra{\supH\sumin(\varphi\circ h)(Z_i)-\supH\sumin\E\bigbra{(\varphi\circ h)(Z_i)}} \\
        &\le\E\bra{\supH\biggabs{\sumin(\varphi\circ h)(Z_i)-\sumin\E\bigbra{(\varphi\circ h)(Z_i)}}} \label{eq:Lemma-ExpectedSupSumBound-eq3} \\
        &=\E\bra{\supH\biggabs{\sumin\Bigrbra{(\varphi\circ h)(Z_i)-\E\bigbra{(\varphi\circ h)(Z_i)}}}} \\
        &\le2\E\bra{\supH\biggabs{\sumin\eps_i(\varphi\circ h)(Z_i)}} \label{eq:Lemma-ExpectedSupSumBound-eq5} \\
        &\le 4L\E\bra{\supH\biggabs{\sumin\eps_i h(Z_i)}}. \label{eq:Lemma-ExpectedSupSumBound-eq6}
    \end{align}
    Inequality \eqref{eq:Lemma-ExpectedSupSumBound-eq3} follows from the triangle inequality, inequality \eqref{eq:Lemma-ExpectedSupSumBound-eq5} follows from a standard symmetrization argument (see \cite{koltchinskii2011oracle}), and inequality \eqref{eq:Lemma-ExpectedSupSumBound-eq6} follows from the Ledoux-Talagrand contraction inequality (see Lemma \ref{lem:ContractionInequality}). The result follows by moving the second term in Equation \eqref{eq:Lemma-ExpectedSupSumBound-eq1} to the right-hand side in the last inequality.
\end{proof}

\section{COMPLEXITY AND HETEROGENEITY MEASURES}\label{app:PolicyClassComplexityMeasures}

In this section, we introduce important quantities of policy class complexity and source heterogeneity that appear in our analysis. All throughout, we let $n=\sumiM\ni$ be the total sample size across sources, $\nC=(\nc)_{\source\in\sourceSet}$ the vector of sample sizes across sources, and $\barn=(\ni/n)_{\source\in\sourceSet}$ the empirical distribution of samples across sources. 

\subsection{Weighted Rademacher Complexity}
Our learning bounds will rely on the following weighted multiple-source generalization of Rademacher complexity introduced in \cite{mohri2019agnostic}.

\begin{definition}[Weighted Rademacher complexity]
    Suppose there is a set of sources $\sourceSet$, where each source $\source\in\sourceSet$ possesses a data-generating distribution $\calDc$ defined over a common space $\calZ$.
    For each source $\source\in\sourceSet$, consider a collection of independently sampled random variables
    $Z_1^\source,\dots,Z_{\ni}^\source\sim\calDc$, and let $Z=\{Z_i^\source\mid \source\in\sourceSet, i\in[\ni]\}$ denote the set of samples across all sources. Additionally, let $\eps=\{\epsij\mid \source\in\sourceSet, i\in[n_\source]\}$ be a corresponding set of independent Rademacher random variables.

    The \textit{empirical weighted Rademacher complexity} of a function class $\calF$ on $\calZ$ given multi-source data $Z$ under fixed mixture weights $\w\in\Delta_\sourceSet$ and sample sizes $\nC$ is
    \begin{align*}
        \frakR_{\w,n_\sourceSet}(\calF;Z)\coloneqq\E\bra{\supF\biggabs{\sumijw\epsij f(\Zij)}\ \Big|\ Z},
    \end{align*}
    where the expectation is taken with respect to the Rademacher random variables $\eps$.
    Additionally, the \textit{weighted Rademacher complexity} of $\calF$ under fixed mixture weights $\w\in\Delta_\sourceSet$ and sample sizes $n_\sourceSet$ is
    \begin{align*}
        \frakR_{\w,\nC}(\calF)\coloneqq\E\bra{\supF\biggabs{\sumijw\epsij f(\Zij)}},
    \end{align*}
    where the expectation is taken with respect to the multi-source random variables $Z$ and the independent Rademacher random variables $\eps$.
\end{definition}

\subsection{Hamming Distance \& Entropy Integral}\label{app:HammingDistance}

We provide additional details on the definition of the entropy integral introduced in Section \ref{sec:ComplexityAndSkewness}.

\begin{definition}[Hamming distance, covering number, and entropy integral]\label{def:HammingDistance-CoveringNumber-EntropyIntegral}
    Consider a policy class $\Pi$ and a multi-source covariate set $x=\x\subset\calX$ across sources $\sourceSet$ with source sample sizes $\nC$. We define the following:
    \begin{enumerate}[(a)]
        \item the Hamming distance between any two policies $\pi_1,\pi_2\in\Pi$ given multi-source covariate set $x$  is
        \begin{equation*}
            \Ham(\pi_1,\pi_2;x)\coloneqq\frac{1}{\sumiM\nc}\sumiM\sumjni\ones\{\pi_1(\xij)\neq\pi_2(\xij)\};
        \end{equation*}
        
        \item an $\epsilon$-cover of $\Pi$ under the Hamming distance given covariate set $x$ is any policy set $\Pi_\epsilon\subset\Pi$ such that for any $\pi\in\Pi$ there exists some $\pi'\in \Pi_\epsilon$ such that $\Ham(\pi,\pi';x)\le\epsilon$;
    
        \item the $\epsilon$-covering number of $\Pi$ under the Hamming distance given covariate set $x$ is
        \begin{equation*}
            N_\Ham(\epsilon,\Pi;x)\coloneqq\min\{|\Pi_\epsilon|\mid \Pi_\epsilon\text{ is an $\epsilon$-cover of $\Pi$ w.r.t.~$\Ham(\cdot,\cdot;x)$}\};
        \end{equation*}
        
        \item the $\epsilon$-covering number of $\Pi$ under the Hamming distance is
        \begin{equation*}
            N_\Ham(\epsilon,\Pi)\coloneqq\sup\{N_\Ham(\epsilon,\Pi;x)\mid x\in\calX_\sourceSet\},
        \end{equation*}
        where $\calX_\sourceSet$ is the set of all covariate sets in $\calX$ across sources $\sourceSet$ with arbitrary sample sizes;
        
        \item the entropy integral of $\Pi$ is
        \begin{equation*}
            \kappa(\Pi)\coloneqq\int_0^1\sqrt{\log N_\Ham(\epsilon^2,\Pi)}d\epsilon.
        \end{equation*}
    \end{enumerate}
\end{definition}

\subsection{\texorpdfstring{$\ell_\lambda$}{ } Distance}\label{app:ellw2Distance}
Throughout our analysis, we will consider the function class
\begin{equation*}
    \calF_\Pi\coloneqq\{Q(\cdot,\pi):\calZ\to\R\mid\pi\in\Pi\},
\end{equation*}
where
\begin{equation*}
    Q(\zij;\pi)\coloneqq\gij(\pi(\xij))
\end{equation*}
for any covariate-score vector $\zij=(\xij,\gij)\in\calZ=\calX\times\R^d$ and $\pi\in\Pi$, where $\gij(a)$ is the $a$-th coordinate of the score vector $\gij$.

\begin{definition}[$\ellw$ distance and covering number]
    Consider a policy class $\Pi$, function class $\calF_\Pi$, and a multi-source covariate-score set $z=\{\zij\mid \source\in\sourceSet,i\in[\nc]\}\subset\calZ$ across sources $\sourceSet$ with source sample sizes $\nC$ and source mixture weights $\w$. We define:
    \begin{enumerate}[(a)]
        \item the $\ellw$ distance with respect to function class $\calF_\Pi$ between any two policies $\pi_1,\pi_2\in\Pi$ given covariate-score set $z$ is
        \begin{equation*}
            \ellw(\pi_1,\pi_2;z)=\sqrt{\frac{\sumijwsq\bigrbra{Q(\zij;\pi_1)-Q(\zij;\pi_2)}^2}{\sup_{\pi_a,\pi_b\in\Pi}\sumijwsq\bigrbra{Q(\zij;\pi_a)-Q(\zij;\pi_b)}^2}};
        \end{equation*}
        
        \item an $\epsilon$-cover of $\Pi$ under the $\ellw$ distance given covariate-score set $z$ is any policy set $\Pi_\epsilon$ such that for any $\pi\in\Pi$ there exists some $\pi'\in\Pi_\epsilon$ such that $\ellw(\pi,\pi';z)\le\epsilon$;
        
        \item the $\epsilon$-covering number of $\Pi$ under the $\ellw$ distance given covariate-score set $z$ is
        \begin{equation*}
            N_{\ellw}(\epsilon,\Pi;z)\coloneqq\min\{|\Pi_\epsilon|\mid \Pi_\epsilon\text{ is an $\epsilon$-cover of $\Pi$ w.r.t.~$\ellw(\cdot,\cdot;z)$}\}.
        \end{equation*}
    \end{enumerate}
\end{definition}

\vspace{1em}
The following lemma relates the covering numbers of the two policy distances we have defined.

\begin{lemma}
Let $z=\{\zij\mid \source\in\sourceSet,i\in[\nc]\}\subset\calZ$ be a multi-source covariate-score set across sources $\sourceSet$ with source sample sizes $\nC$ and source mixture weights $\w$. For any $\epsilon>0$,
\begin{equation*}
    N_{\ellw}(\epsilon,\Pi;z)\le N_{\Ham}(\epsilon^2,\Pi).
\end{equation*}
\end{lemma}
\begin{proof}
    Fix $\epsilon>0$. 
    Without loss of generality, we assume $N_{\Ham}(\epsilon^2,\Pi)<\infty$, otherwise the result trivially holds. Let $\Pi_\epsilon=\{\pi_1,\dots,\pi_{N_0}\}$ be a corresponding Hamming $\epsilon^2$-cover of $\Pi$.

    Consider any arbitrary $\pi\in\Pi$. By definition, there exists a $\pi'\in \Pi_\epsilon$ such that for any multi-source covariate set $\tildex=\{\tildexij\mid \source\in\sourceSet, i\in[\tilden_c]\}$ with any given sample sizes $\tilden_c>0$ the following holds:
    \begin{equation*}
        \Ham(\pi,\pi';\tildex)=\frac{1}{\tilden}\sumiM\sum_{i=1}^{\tilden_c}\ones\{\pi(\tildexij)\neq\pi'(\tildexij)\}\le\epsilon^2,
    \end{equation*}
    where $\tilden=\sumiM\tilden_c$. Using this pair of policies $\pi, \pi'$ we generate an augmented data set $\tildez$ from $z$ as follows.
    Let $m$ be a positive integer and define $\tildez$ to be a collection of multiple copies of all covariate-score tuples $\zij\in z$, where each $\zij$ appears
    \begin{equation*}
        \tildenij=\left\lceil\frac{m\cdot\fracwinisq\bigrbra{Q(\zij;\pi)-Q(\zij;\pi')}^2}{\sup_{\pi_a,\pi_b}\sumijwsq\bigrbra{Q(\zij;\pi_a)-Q(\zij;\pi_b)}^2}\right\rceil    
    \end{equation*}
    times in $\tildez$. Therefore, the source sample sizes in this augmented data set are $\tilden_c=\sumjni\tildenij$ and the total sample size is $\tilden=\sumij\tildenij$. The total sample size is bounded as
    \begin{align*}
        \tilden&=\sumij\left\lceil\frac{m\cdot\fracwinisq\bigrbra{Q(\zij;\pi)-Q(\zij;\pi')}^2}{\sup_{\pi_a,\pi_b}\sumijwsq\bigrbra{Q(\zij;\pi_a)-Q(\zij;\pi_b)}^2}\right\rceil \\
        &\le\sumij\rbra{\frac{m\cdot\fracwinisq\bigrbra{Q(\zij;\pi)-Q(\zij;\pi')}^2}{\sup_{\pi_a,\pi_b}\sumijwsq\bigrbra{Q(\zij;\pi_a)-Q(\zij;\pi_b)}^2}+1} \\
        &\le\frac{m\cdot\sumijwsq\bigrbra{Q(\zij;\pi)-Q(\zij;\pi')}^2}{\sup_{\pi_a,\pi_b}\sumijwsq\bigrbra{Q(\zij;\pi_a)-Q(\zij;\pi_b)}^2}+n\le m+n.
    \end{align*}
    Then, we have
    \begin{align*}
        \Ham(\pi,\pi';\tildez)&=\frac{1}{\tilde{n}}\sumiM\sum_{i=1}^{\tilden_c}\ones\{\pi(\xij)\neq\pi'(\xij)\} \\
        &=\frac{1}{\tilden}\sumij\tildenij\cdot\ones\{\pi(\xij)\neq\pi'(\xij)\} \\
        &\ge\frac{1}{\tilden}\sumij\frac{m\cdot\fracwinisq\bigrbra{Q(\zij;\pi)-Q(\zij;\pi')}^2}{\sup_{\pi_a,\pi_b}\sumijwsq\bigrbra{Q(\zij;\pi_a)-Q(\zij;\pi_b)}^2}\ones\{\pi(\xij)\neq\pi'(\xij)\} \\
        &=\frac{m}{\tilden}\sumij\frac{\fracwinisq\bigrbra{Q(\zij;\pi)-Q(\zij;\pi')}^2}{\sup_{\pi_a,\pi_b}\sumijwsq\bigrbra{Q(\zij;\pi_a)-Q(\zij;\pi_b)}^2} \\
        &\ge\frac{m}{m+n}\ellw^2(\pi,\pi';z).
    \end{align*}
    
    This implies that
    \begin{equation*}
        \ellw(\pi,\pi';z)\le\sqrt{\frac{m+n}{m}\Ham(\pi,\pi';\tilde{z})}\le\sqrt{1+\frac{n}{m}}\cdot\epsilon.
    \end{equation*}
    Letting $m\to\infty$ yields $\ellw(\pi,\pi';z)\le\epsilon$. This establishes that for any $\pi\in\Pi$, there exists a $\pi'\in \Pi_\epsilon$ such that $\ellw(\pi,\pi';z)\le\epsilon$, and thus $N_{\ellw}(\epsilon,\Pi;z)\le N_{\Ham}(\epsilon^2,\Pi)$.
\end{proof}

\subsection{Mixture Skewness}\label{app:Preliminaries-sourceDistributionSkewness}

An important quantity that arises in our analysis is
\begin{equation*}
    \sumiM\fracwisqbarni=\opsE_{\source\sim\w}\bra{\fracwibarni},
\end{equation*}
which captures a measure of the imbalance of the mixture weight distribution $\w$ relative to the empirical distribution of samples $\barn$. The following result makes this interpretation more clear:
\begin{align*}
    \sumiM\fracwisqbarni&=\sumiM\fracwisqbarni+\sumiM\barni-2\sumiM\wi+1 \\
    &=\sumiM\rbra{\fracwisqbarni+\frac{\barni^2}{\barni}-\frac{2\wi\barni}{\barni}}+1 \\
    &=\sumiM\frac{(\wi-\barni)^2}{\barni} + 1 \\
    &=\chi^2(\w||\barn)+1,
\end{align*}
where $\chi^2(\w||\barn)$ is the chi-squared divergence from $\barn$ to $\w$.
Following \cite{mohri2019agnostic}, this quantity is referred to as skewness.
\begin{definition}[Skewness]
    The \textit{skewness} of a given set of mixture weights $\w$ relative to the empirical distribution of samples $\barn\coloneqq(\ni/n)_{\source\in\sourceSet}$ is
    \begin{equation*}
        \skewness\coloneqq 1+\chi^2(\w||\barn),
    \end{equation*}
    where $\chi^2(\w||\barn)$ is the chi-squared divergence of $\w$ from $\barn$. Additionally, the \textit{mixture-agnostic skewness} is
    \begin{equation*}
        \skewnessW\coloneqq\maxwW\skewness.
    \end{equation*}
\end{definition}

\section{BOUNDING MIXTURE REGRET}\label{app:BoundingGlobalRegret}

\subsection{Preliminaries}

\subsubsection{Function Classes}
As mentioned in Appendix \ref{app:ellw2Distance}, the function class we will be considering in our analysis is
\begin{equation}\label{eq:FunctionClass}
    \calF_\Pi\coloneqq\{Q(\cdot;\pi):\calZ\to\R\mid\pi\in\Pi\},
\end{equation}
where
\begin{equation}\label{eq:QFunction}
    Q(\zij;\pi)\coloneqq\gij(\pi(\xij)),
\end{equation}
for any covariate-score vector $\zij=(\xij,\gij)\in\calZ=\calX\times\R^d$ and $\pi\in\Pi$, where $\gij(a)$ is the $a$-th coordinate of the score vector $\gij$. It will also be useful to consider the Minkowski difference of $\calF_\Pi$ with itself,
\begin{equation}\label{eq:DeltaFunctionClass}
    \Delta\calF_\Pi\coloneqq\calF_\Pi-\calF_\Pi=\{\Delta(\cdot;\pia,\pib):\calZ\to\R\mid\pia,\pib\in\Pi\},
\end{equation}
where
\begin{equation}\label{eq:DeltaQFunction}
    \Delta(\zij;\pia,\pib)\coloneqq Q(\zij;\pia)-Q(\zij;\pib)=\gij(\pia(\xij))-\gij(\pib(\xij)),
\end{equation}
for any $\zij=(\xij,\gij)\in\calZ$ and $\pia,\pib\in\Pi$.

\subsubsection{Policy Value Estimators}\label{app:Preliminaries-PolicyValueEstimators}

\paragraph{Augmented Inverse Propensity Weighted Scores}
As discussed in Section \ref{sec:Approach-PolicyValueEstimators},
we used propensity-weighted scores to estimate the policy values.
We considered the construction of the oracle local AIPW scores
\begin{equation*}
    \Gij(a)=\mui(\Xij;a)+\bigrbra{\Yij(a)-\mui(\Xij;a)}\oi(\Xij;a)\ones\{\Aij=a\}
\end{equation*}
for each $a\in\calA$, constructed from the available observable samples $(\Xij,\Aij,\Yij)$ taken from the partially observable counterfactual samples $(\Xij,\Aij,\Yij(a_1),\dots,\Yij(a_d))\sim\calDcec$ for $i\in[\nc]$.
Similarly, we discussed the construction of the approximate local AIPW scores
\begin{equation*}
    \Ghatij(a)=\muhati(\Xij;a)+\bigrbra{\Yij-\muhati(\Xij;a)}\ohati(\Xij;a)\ones\{\Aij=a\}
\end{equation*}
for each $a\in\calA$, given fixed estimates $\muhati$ and $\ohati$ of $\mui$ and $\oi$, respectively. In practice, we use cross-fitting to make the estimates fixed and independent relative to the data on which they are evaluated (see further discussion in Appendix \ref{app:AdditionalAlgorithmDetails}). Note that only this second set of AIPW scores can be constructed from the observed data. The first set is only ``constructed" for analytic purposes for our proofs.

\paragraph{Policy Value Estimates and Policy Value Difference Estimates}

Using the source data and the constructed AIPW scores, we let $\Zij=(\Xij,\Gij(a_1),\dots,\Gij(a_d))$ and $\Zhatij=(\Xij,\Ghatij(a_1),\dots,\Ghatij(a_d))$ for each $i\in[\nc]$. Using these constructed vectors, we define the oracle and approximate mixture policy value estimates of $\Qw(\pi)$, respectively, as
\begin{align*}
    \Qtildew(\pi)&\coloneqq\sumijw\Gij(\pi(\Xij))=\sumijw Q(\Zij;\pi), \\
    \Qhatw(\pi)&\coloneqq\sumijw\Ghatij(\pi(\Xij))=\sumijw Q(\Zhatij;\pi),
\end{align*}
for any $\pi\in\Pi$, where we use the function class defined in Equation \ref{eq:QFunction} for these representations, which we will use throughout our proofs for notational convenience. It will also be very useful to define the following corresponding \textit{policy value difference} quantities:
\begin{align*}
    \Deltaw(\pia,\pib)&\coloneqq\Qw(\pia)-\Qw(\pib), \\
    \Deltatildew(\pia,\pib)&\coloneqq\Qtildew(\pia)-\Qtildew(\pib)=\sumijw\Delta(\Zij;\pia,\pib), \\
    \Deltahatw(\pia,\pib)&\coloneqq\Qhatw(\pia)-\Qhatw(\pib)=\sumijw\Delta(\Zhatij;\pia,\pib),
\end{align*}
for any $\pia,\pib\in\Pi$, where we use the function class defined in Equation \eqref{eq:DeltaQFunction} for these representations, which we will use throughout our proofs for notational convenience.

\paragraph{Unbiased Estimates}
The following result can be used to readily show that the oracle estimator for the mixture policy value is unbiased.
\begin{lemma}\label{lem:ExpectedOracleEqualsLocalPolicyValue}
    Suppose Assumption \ref{ass:dgp} holds.
    For any $\pi\in\Pi$,
    \begin{equation*}
        \opsE_{\Zi\sim\calDcec}\regbra{\Gi(\pi(\Xi))}=\Qi(\pi)\quad\text{and}\quad \opsE_{\source\sim\w}\opsE_{\Zi\sim\calDcec}[\Gi(\pi(\Xi))]=\Qw(\pi),
    \end{equation*}
    where $\Zi=(\Xi,\Ai,\Yi(a_1),\dots,\Yi(a_d))\sim\calDcec$.
\end{lemma}
\begin{proof}
    First, observe that for any $a\in\calA$,
    \begin{align*}
        \Gi(a)&=\mui(\Xi;a)+\rbra{\Yi(\Ai)-\mui(\Xi;a)}\oi(\Xi;a)\ones\{\Ai=a\} \\
        &=\mui(\Xi;a)+\rbra{\Yi(a)-\mui(\Xi;a)}\oi(\Xi;a)\ones\{\Ai=a\}
    \end{align*}
    and so for any $(\Xi,\Ai,\vecYi)=(\Xi,\Ai,\Yi(a_1),\dots,\Yi(a_d))\sim\calDcec$,
    \begin{align*}
        \E_{\Ai,\vecYi}\bra{\Gi(a)\mid\Xi}&=\mui(\Xi;a)+\E_{\Ai,\vecYi}\bra{\rbra{\Yi(a)-\mui(\Xi;a)}\oi(\Xi;a)\ones\{\Ai=a\}\mid\Xi}  \\
        &=\mui(\Xi;a)+\E_{\vecYi}\bra{\Yi(a)-\mui(\Xi;a)\mid\Xi}\cdot\E_{\Ai}\bra{\oi(\Xi;a)\ones\{\Ai=a\}\mid\Xi} \\
        &=\mui(\Xi;a)+\E_{\vecYi}\bra{\Yi(a)-\mui(\Xi;a)\mid\Xi}\cdot\oi(\Xi;a)\ei(\Xi;a) \\
        &=\mui(\Xi;a)+\E_{\vecYi}\bra{\Yi(a)\mid\Xi}-\mui(\Xi;a) \\
        &=\E_{\vecYi}\bra{\Yi(a)\mid\Xi}.
    \end{align*}
    The second equality follows from the unconfoundedness assumption stated in Assumption \ref{ass:dgp}. This immediately implies that
    \begin{align*}
        \opsE_{\Zi\sim\calDcec}\regbra{\Gi(\pi(\Xi))}
        &=\opsE_{\Zi\sim\calDcec}\bra{\sumaA\ones\{\pi(\Xi)=a\}\Gi(a)} \\
        &=\E_{\Xi}\bra{\sumaA\ones\{\pi(\Xi)=a\}\E_{\Ai,\vecYi}\bra{\Gi(a)\mid\Xi}} \\
        &=\E_{\Xi}\bra{\sumaA\ones\{\pi(\Xi)=a\}\E_{\Ai,\vecYi}\bra{\Yi(a)\mid\Xi}} \\
        &=\E_{\Xi}\bra{\E_{\Ai,\vecYi}\bra{\Yi(\pi(\Xi))\mid\Xi}} \\
        &=\opsE_{\Zi\sim\calDcec}\bra{\Yi(\pi(\Xi)} \\
        &=\Qi(\pi),
    \end{align*}
    and thus,
    \begin{align*}
        \opsE_{\source\sim\w}\opsE_{\Zi\sim\calDcec}\regbra{\Gi(\pi(\Xi))}=\opsE_{\source\sim\w}\opsE_{\Zi\sim\calDc}\regbra{\Yi(\pi(\Xi))}=\Qw(\pi).
    \end{align*}
\end{proof}

\subsubsection{Data-generating Distributions and Sufficient Statistics}\label{app:BoundingGlobalRegret-Preliminaries-DGP&SS}

Note that the contexts and AIPW scores are \textit{sufficient statistics} for the corresponding oracle and approximate estimators of the policy values. Moreover, our results will mostly directly depend on properties of the sufficient statistics.
% (e.g., AIPW score range and variance).
Therefore, it will sometimes be simpler for our analysis to work with the distribution of the sufficient statistics directly. For any $(\Xi,\Ai,\Yi(a_1),\dots,\Yi(a_d))\sim\calDcec$, let $(\Xi,\Gi(a_1),\dots,\Gi(a_d))$
be the sufficient statistic of contexts and oracle AIPW scores, and denote its induced distribution as $\smash[t]{\tildecalDc}$ defined over $\calZ=\calX\times\R^d$.

For simplicity and without loss of generality, when proving results that only involve the contexts and AIPW scores, we will assume the data is sampled from the distributions of the sufficient statistics, e.g., $\Zi\sim\tildecalDc$. When we have a discussion involving constructing the AIPW scores from the observable data, we will be more careful about the source distributions and typically assume the data is sampled from the complete data-generating distributions, e.g., $\Zi\sim\calDcec$.

\subsection{Bounding Weighted Rademacher Complexity}

First, to simplify our analysis, we can easily bound the weighted Rademacher complexity of $\Delta\calF_\Pi$ by that of $\calF_\Pi$ as follows.
\begin{lemma}\label{lem:RademacherComplexityDifferenceBound}
    \begin{align*}
        \frakR_{\w,\nC}(\Delta\calF_\Pi)\le2\frakRwnF.
    \end{align*}
\end{lemma}
\begin{proof}
    \begin{align*}
        \frakR_{\w,\nC}(\Delta\calF_\Pi)&=\E\bra{\supPiab\Biggabs{\sumijw\epsij\Delta(\Zij;\pia,\pib)}} \\
        &=\E\bra{\supPiab\Biggabs{\sumijw\epsij\Bigrbra{Q(\Zij;\pia)-Q(\Zij;\pib)}}} \\
        &\le\E\bra{\supPiab\Biggabs{\sumijw\epsij Q(\Zij;\pia)}+\Biggabs{\sumijw\epsij Q(\Zij;\pib)}} \\
        &=2\E\bra{\supPi\Biggabs{\sumijw\epsij Q(\Zij;\pi)}} \\
        &=2\frakR_{\w,\nC}(\calF_\Pi).
    \end{align*}
\end{proof}

Therefore, we can simply focus on bounding the weighted Rademacher complexity of $\calF_\Pi$.

\begin{proposition}\label{prop:WeightedRademacherBound}
Suppose Assumptions \ref{ass:dgp} and \ref{ass:LocalDataSizeScaling} hold. Then,
\begin{equation*}
    \frakR_{\w,\nC}(\calF_\Pi)\le\rbra{14+6\kappa(\Pi)}\sqrtfracVwn + \littleo{\sqrt{\frac{\skewness}{n}}},
\end{equation*}
where
\begin{equation*}
    \Vwn=\supPiab\sumiM\fracwisqbarni\opsE_{\Zi\sim\tildecalDc}\regbra{\Delta^2(\Zi;\pia,\pib)}.
\end{equation*}
\end{proposition}
\begin{proof}
We follow a chaining argument to bound the weighted Rademacher complexity of $\calF_\Pi$.
\vspace{-.5em}
\paragraph{Constructing the policy approximation chain.}
First, for each source $\source\in\sourceSet$, let $Z_1^c,\dots,Z_{\ni}^\source$ be $\nc$ independent random variables sampled from $\smash[t]{\tildecalDc}$, where each $\Zij=(\Xij,\smash[t]{\vecGij})\in\calZ=\calX\times\R^d$. Additionally, let $Z=\{\Zij\mid \source\in\sourceSet, i\in[n_c]\}$ represent the corresponding set of samples across all sources.

Next, set $K=\lceil\log_2 n\rceil$. We will construct a sequence $\{\Psi_k:\Pi\to\Pi\}_{k=0}^K$ of policy approximation operators that satisfies the following properties. For any $k=0,\dots,K$,
\begin{itemize}
    \item[(P1)] $\maxPi\ellw(\Psi_{k+1}(\pi),\Psi_k(\pi);Z)\le \epsilon_k\coloneqq 2^{-k}$
    
    \item[(P2)] $\abs{\{\Psi_k(\pi)\mid\pi\in\Pi\}}\le N_{\ellw}(\epsilon_k,\Pi;Z)$
\end{itemize}
We use the notational shorthand that $\Psi_{K+1}(\pi)=\pi$ for any $\pi\in\Pi$. We will construct the policy approximation chain via a backward recursion scheme. First, let $\Pi_k$ denote the smallest $\epsilon_k$-covering set of $\Pi$ under the $\ellw$ distance given data $Z$. Note, in particular, that $\abs{\Pi_0}=1$ since the $\ellw$ distance is never more than 1 and so any single policy is enough to 1-cover all policies in $\Pi$. Then, the backward recursion is as follows: for any $\pi\in\Pi$,
\begin{enumerate}
    \item define $\Psi_K(\pi)=\argmin_{\pi'\in\Pi_K}\ellw(\pi,\pi';Z)$;
    
    \item for each $k=K-1,\dots,1$, define $\Psi_k(\pi)=\argmin_{\pi'\in\Pi_k}\ellw(\Psi_{k+1}(\pi),\pi';Z)$;
    
    \item define $\Psi_0(\pi)\equiv 0$.
\end{enumerate}

Note that although $\Psi_0(\pi)$ is not in $\Pi$, it can still serve as a $1$-cover of $\Pi$ since the $\ellw$ distance is always bounded by 1. Before proceeding, we check that each of the stated desired properties of the constructed operator chain is satisfied:
\begin{itemize}
    \item[(P1)] Pick any $\pi\in\Pi$. Clearly, $\Psi_{k+1}(\pi)\in\Pi$. Then, by construction of $\Pi_k$, there exists a $\pi'\in\Pi_k$ such that $\ellw(\Psi_{k+1}(\pi),\pi';Z)\le\epsilon_k$. Therefore, by construction of $\Psi_k(\pi)$, we have $\ellw(\Psi_{k+1}(\pi),\Psi_k(\pi);Z)\le\ellw(\Psi_{k+1}(\pi),\pi';Z)\le\epsilon_k$.
    
    \item[(P2)] By construction of $\Psi_k$, we have that $\Psi_k(\pi)\in\Pi_k$ for every $\pi\in\Pi$. Therefore, $\abs{\{\Psi_k(\pi)\mid\pi\in\Pi\}}\le\abs{\Pi_k}=N_{\ellw}(\epsilon_k,\Pi;Z)$.
\end{itemize}

Thus, the constructed chain satisfies the desired properties.
Next, we observe that since $\Psi_0(\pi)\equiv 0$, we have that $Q(\Zij;\Psi_0(\pi))=0$ and
\begin{align*}
    Q(\Zij;\pi)&=Q(\Zij;\pi)-Q(\Zij;\Psi_0(\pi)) \\
    &=Q(\Zij;\pi)-Q(\Zij;\Psi_K(\pi)) + \sumkK Q(\Zij;\Psi_k(\pi)-Q(\Zij;\Psi_{k-1}(\pi)) \\
    &=\Delta(\Zij;\pi,\Psi_K(\pi)) + \sumkK\Delta(\Zij;\Psi_k(\pi),\Psi_{k-1}(\pi))
\end{align*}
Therefore, we can decompose the weighted Rademacher complexity of $\calF_\Pi$ as follows:
\begin{align*}
    \frakRwnF\le&\E\bra{\supPi\Biggabs{\sumijw\epsij\Delta(\Zij;\pi,\Psi_K(\pi))}} \\
    &+\E\bra{\supPi\Biggabs{\sumijw\epsij\biggrbra{\sumkK \Delta(\Zij;\Psi_k(\pi),\Psi_{k-1}(\pi))}}}
\end{align*}
We will obtain bounds separately for these two terms, which we refer to as the \textit{negligible regime} term and the \textit{effective regime} term, respectively.

\paragraph{Bounding the negligible regime.}
First, we define the following quantities:
\begin{equation*}
    \BwW\coloneqq\sup_{\pi_a,\pi_b\in\Pi}\sumij\fracwinisq\Delta^2(\Zij;\pia,\pib)\quad\text{and}\quad\Bw\coloneqq\E\bra{\BwW}
\end{equation*}

Given any realization of independent Rademacher random variables $\eps=\{\epsij\mid \source\in\sourceSet, i\in[\nc]\}$ and multi-source data $Z$, by the Cauchy-Schwarz inequality, we have
\begin{align*}
    \biggabs{\sumijw\epsij\Delta(\Zij;\pi,\Psi_K(\pi)}&\le\sqrt{\sumij(\epsij)^2}\cdot\sqrt{\sumij\fracwinisq\Delta^2(\Zij;\pi,\Psi_K(\pi)} \\
    &=\sqrt{n}\cdot\sqrt{\BwW}\ell_2(\pi,\Psi_K(\pi);Z) \\
    &\le\sqrt{n\BwW}\epsilon_K \\
    &\le\sqrt{\frac{\BwW}{n}}.
\end{align*}
Then, by Jensen's inequality,
\begin{align*}
    \E\bra{\supPi\biggabs{\sumijw\epsij\Delta(\Zij;\pi,\Psi_K(\pi)}}&\le\E\bra{\sqrt{\frac{\BwW}{n}}}\le\sqrt{\frac{\Bw}{n}}.
\end{align*}

\paragraph{Bounding the effective regime.}

For any $k\in[K]$, let $$t_{k,\delta}=\sqrt{\BwW}\epsilon_k\tau_{k,\delta}$$ where $\tau_{k,\delta}>0$ is some constant to be specified later. By Hoeffding's inequality (in Lemma \ref{lem:HoeffdingInequality}),
\begin{align*}
    &\P\rbra{\biggabs{\sumijw\epsij\Delta(\Zij;\Psi_k(\pi),\Psi_{k-1}(\pi))}>t_{k,\delta}\Bigmid Z} \\
    &\le 2\exp\rbra{-\frac{t_{k,\delta}^2}{2\sumij\fracwinisq\Delta^2(\Zij;\Psi_k(\pi),\Psi_{k-1}(\pi))}} \\
    &=2\exp\rbra{-\frac{t_{k,\delta}^2}{2\BwW\ell_2^2(\Psi_k(\pi),\Psi_{k-1}(\pi);Z)}} \\
    &\le2\exp\rbra{-\frac{t_{k,\delta}^2}{2\BwW\epsilon_{k-1}^2}} \\
    &=2\exp\rbra{-\frac{t_{k,\delta}^2}{8\BwW\epsilon_k^2}} \\
    &=2\exp\rbra{-\frac{\tau_{k,\delta}^2}{8}}.
\end{align*}
Here, we used the fact that $\epsilon_{k-1}=2\epsilon_k$. Setting
\begin{equation*}
    \tau_{k,\delta}=\sqrt{8\log\rbra{\frac{\pi^2k^2}{3\delta}N_{\ellw}(\epsilon_k,\Pi;Z)}}
\end{equation*}
and applying a union bound over the policy space, we obtain
\begin{align*}
    &\P\rbra{\supPi\biggabs{\sumijw\epsij\Delta(\Zij;\Psi_k(\pi),\Psi_{k-1}(\pi))}>t_{k,\delta}\Bigmid Z} \\
    &\le2\abs{\Pi_k}\cdot\exp\rbra{-\frac{\tau_{k,\delta}^2}{8}} \\
    &\le2N_{\ell_2}(\epsilon_k,\Pi;Z)\cdot\exp\rbra{-\frac{\tau_{k,\delta}^2}{8}} \\
    &=\frac{6\delta}{\pi^2k^2}.
\end{align*}
By a further union bound over $k\in[K]$, we obtain
\begin{align*}
    &\P\rbra{\supPi\biggabs{\sumijw\epsij\biggrbra{\sumkK \Delta(\Zij;\Psi_k(\pi),\Psi_{k-1}(\pi))}}>\sumkK t_{k,\delta}\Bigmid Z} \\
    &\le\sumkK\P\rbra{\supPi\biggabs{\sumijw\epsij\Delta(\Zij;\Psi_k(\pi),\Psi_{k-1}(\pi))}>t_{k,\delta}\Bigmid Z} \\
    &\le\sumkK\frac{6\delta}{\pi^2k^2} \le\delta.
\end{align*}
Therefore, given multi-source data $Z$, with probability at least $1-\delta$, we have
\begin{align*}
    &\supPi\Biggabs{\sumijw\epsij\biggrbra{\sumkK \Delta(\Zij;\Psi_k(\pi),\Psi_{k-1}(\pi))}} \\
    &\le\sumkK t_{k,\delta} \\
    &=\sqrt{\BwW}\sumkK\epsilon_k\sqrt{8\log\rbra{\frac{\pi^2k^2}{3\delta}N_{\ell_2}(\epsilon_k,\Pi;Z)}} \\
    &=\sqrt{\BwW}\sumkK\epsilon_k\rbra{\sqrt{8\log\frac{\pi^2}{3\delta}+16\log k + 8\log N_{\ell_2}(\epsilon_k,\Pi;Z)}} \\
    &\le\sqrt{\BwW}\sumkK\epsilon_k\rbra{\sqrt{8\log\frac{\pi^2}{3\delta}}+\sqrt{16\log k} + \sqrt{8\log N_{\ell_2}(\epsilon_k,\Pi;Z)}} \\
    &\le\sqrt{\BwW}\sumkK\epsilon_k\rbra{\sqrt{8\log(4/\delta)}+\sqrt{16\log k} + \sqrt{8\log N_{\Ham}(\epsilon_k,\Pi)}} \\
    &\le\sqrt{\BwW}\rbra{\sqrt{8\log(4/\delta)}+2+\sqrt{8}\sum_{k=1}^\infty\epsilon_k\sqrt{\log N_{\Ham}(\epsilon_k,\Pi)}} \\
    &\le\sqrt{\BwW}\rbra{\sqrt{8\log(4/\delta)}+2+\sqrt{8}\kappa(\Pi)}.
\end{align*}
Next, we turn this high-probability bound into a bound on the conditional expectation. First, let $F_R(\cdot\mid Z)$ be the cumulative distribution of the random variable
\begin{equation*}
    R\coloneqq \supPi\Biggabs{\sumijw\epsij\biggrbra{\sumkK \Delta(\Zij;\Psi_k(\pi),\Psi_{k-1}(\pi))}}
\end{equation*}
conditional on $Z$. Above, we have shown that
\begin{equation*}
    1-F_R\rbra{\sqrt{\BwW}\rbra{\sqrt{8\log(4/\delta)}+2+\sqrt{8}\kappa(\Pi)}\ \big|\ Z}\le\delta.
\end{equation*}

For any non-negative integer $l$, let $\Delta_l=\sqrt{\BwW}\cdot\bigrbra{\sqrt{8\log(4/\delta_l)}+2+\sqrt{8}\kappa(\Pi)}$ where $\delta_l=2^{-l}$.
Since $R$ is non-negative, we can compute and upper bound the conditional expectation of $R$ given $Z$ as follows:
\begin{align*}
    &\E\bra{\supPi\Biggabs{\sumijw\epsij\biggrbra{\sumkK \Delta(\Zij;\Psi_k(\pi),\Psi_{k-1}(\pi))}}\Bigmid Z} \\
    &=\int_0^\infty\rbra{1-F_R(r|Z)}dr \\
    &\le\sum_{l=0}^\infty\rbra{1-F_R(\Delta_l|Z)}\cdot\Delta_l \\
    &\le\sum_{l=0}^\infty\delta_l\cdot\Delta_l \\
    &=\sum_{l=0}^\infty 2^{-l}\cdot \sqrt{\BwW}\rbra{\sqrt{8(l+2)\log 2}+2+\sqrt{8}\kappa(\Pi)} \\
    &\le\sqrt{\BwW}\rbra{4\sqrt{8\log 2}+4+2\sqrt{8}\kappa(\Pi)} \\
    &\le\sqrt{\BwW}\rbra{14+6\kappa(\Pi)}.
\end{align*}
Taking the expectation with respect to $Z$ and using Jensen's inequality, we obtain
\begin{align*}
    &\E\bra{\supPi\Biggabs{\sumijw\epsij\biggrbra{\sumkK \Delta(\Zij;\Psi_k(\pi),\Psi_{k-1}(\pi))}}} \\
    &\le\rbra{14+6\kappa(\Pi)}\E\bra{\sqrt{\BwW}} \\
    &\le\rbra{14+6\kappa(\Pi)}\sqrt{\Bw}.
\end{align*}

\paragraph{Refining the upper bound.}
We see that the quantity $\Bw$ appears in the bounds for the negligible and effective regimes. Therefore, the task is to find an appropriate bound for this quanitity.
One could easily bound $\Bw$ using worst-case bounds on the AIPW element. Instead, we use Lemma \ref{lem:ExpectedSupSumBound} to get a more refined bound on $\Bw$.

To make use this result, we first identify the set of independent random variables $\tildeZij=T(\Zij)=(\Xij,\fracwini\vecGij)$ for each $\source\in\sourceSet$ and $i\in[\nc]$ and the function class $\calH=\{\Delta(\cdot;\pia,\pib)\mid\pia,\pib\in\Pi\}$. We also identify the Lipschitz function $\varphi:u\mapsto u^2$ defined over the set $\calU\subset\R$ containing all possible outputs of any $\Delta(\cdot;\pia,\pib)$ given any realization of $\tildeZij$ for any $\source\in\sourceSet$ as input. To further capture this domain, note that by the boundedness and overlap assumptions in Assumption \ref{ass:dgp}, it is easy to verify that there exists some $U>0$ for all $\source\in\sourceSet$ such that $|\Gij(a)|\le U$ for any $a\in\calA$ and any realization of $\Zij$. This implies that
\begin{align*}
    \regabs{\Delta(\tildeZij;\pia,\pib)}&=\fracwini\bigabs{\Gij(\pia(\Xij))-\Gij(\pib(\Xij))}\le 2U\fracwini,
\end{align*}
for any realization of $\tildeZij$ and any $\pia,\pib\in\Pi$. Moreover, note that
\begin{align}
    \fracwini&\le\sqrt{\sumiM\fracwinisq}\le\frac{1}{\sqrt{\min_{\source\in\sourceSet}\nc}}\sqrt{\sumiM\fracwisqni}=\frac{1}{\sqrt{\min_{\source\in\sourceSet}\nc}}\sqrtfracskewnessn\eqqcolon\swnC \label{eq:sourceDistributionRatioBound}
\end{align}
Therefore, $\calU\subset[-2U\swnC,2U\swnC]$, and thus, for any $u,v\in\calU$, we have that
\begin{align*}
    \regabs{\varphi(u)-\varphi(v)}&=\regabs{u^2-v^2}=\abs{u+v}\cdot\abs{u-v}\le4U\swnC\abs{u-v}.
\end{align*}
Therefore, $L=4U\swnC$ is a valid Lipschitz constant for $\varphi$. Then, through these identifications, Lemma \ref{lem:ExpectedSupSumBound} guarantees the following upper bound:
\begin{align*}
    \Bw&=\E\bra{\supPiab\sumij\fracwinisq\Delta^2(\Zij;\pia,\pib)} \\
    &=\E\bra{\supPiab\sumij\varphi\circ\Delta(\tildeZij;\pia,\pib)} \\
    &\le\supPiab\sumij\E\bra{\varphi\circ\Delta(\tildeZij;\pia,\pib)} + 16U\swnC\E\bra{\supPiab\biggabs{\sumij\epsij\Delta(\tildeZij;\pia,\pib)}} \\
    &=\supPiab\sumij\fracwinisq\E\bra{\Delta^2(\Zij;\pia,\pib)} + 16U\swnC\E\bra{\supPiab\biggabs{\sumijw\epsij\Delta(\Zij;\pia,\pib)}} \\
    &=\supPiab\sumiM\fracwisqni\E\bra{\Delta^2(\Zij;\pia,\pib)} + 16U\frakRwn(\Delta\calF_\Pi)\swnC \\
    &\le\supPiab\sumiM\fracwisqni\E\bra{\Delta^2(\Zij;\pia,\pib)} + 32U\frakRwnF\swnC \\
    &=\fracVwn + 32U\frakRwnF\swnC.
\end{align*}

\paragraph{Establishing the skewness bound.}
Before proceeding, note that by the local data size scaling assumption stated in Assumption \ref{ass:LocalDataSizeScaling}, $\nc=\Omega(\nu_c(n))$ for some increasing function $\nu_c$ for any $\source\in\sourceSet$. This immediately implies that $\swnC$ is dominated as
\begin{align*}
    \swnC=\frac{1}{\sqrt{\min_{\source\in\sourceSet}\nc}}\sqrtfracskewnessn\le\littleo{\sqrtfracskewnessn}.
\end{align*}

\paragraph{Combining results.}
Thus, combining the bounds for the negligible and effective regime and including the refined bound and the skewness bound, we have
\begin{align*}
    &\frakRwnF \\
    &\le\sqrt{\frac{\Bw}{n}}+\rbra{14+6\kappa(\Pi)}\sqrt{\Bw} \\
    &\le\sqrt{\frac{\Vwn}{n^2}+32U\frakRwnF\frac{\swnC}{n}} + (14+6\kappa(\Pi))\sqrt{\fracVwn + 32U\frakRwnF\swnC} \\
    &\le\sqrt{\frac{\Vwn}{n^2}} + \sqrt{32U\frakRwnF\frac{\swnC}{n}} + (14+6\kappa(\Pi))\rbra{\sqrtfracVwn + \sqrt{32U\frakRwnF\swnC}} \\
    &\le(14+6\kappa(\Pi))\sqrtfracVwn + \sqrt{\frac{\Vwn}{n^2}} + \bigO{\sqrt{\frakRwnF\swnC}} \numberthis \label{eq:prop-WeightedRademacherBound-eq1}
\end{align*}
This gives an upper bound on $\frakRwnF$ in terms of itself. To decouple this dependence, we express
\begin{align*}
    \frakRwnF&\le \bigO{\sqrtfracVwn} + \bigO{\sqrt{\frakRwnF\swnC}} \\
    &\le A_1\sqrtfracVwn + A_2\sqrt{\frakRwnF\swnC} \numberthis \label{eq:prop-WeightedRademacherBound-eq2}
\end{align*}
for some constants $A_1, A_2$, and we split this inequality into the following two exhaustive cases.

\textit{\underline{Case 1}}: $A_2\sqrt{\swnC}\le\frac{1}{2}\sqrt{\frakRwnF}$

In this case, we can bound the second term in the right-hand side of inequality \eqref{eq:prop-WeightedRademacherBound-eq2} to get
\begin{align*}
    \frakRwnF\le A_1\sqrtfracVwn + \frac{\frakRwnF}{2},
\end{align*}
and so
\begin{align}\label{eq:prop-WeightedRademacherBound-eq3}
    \frakRwnF\le 2A_1\sqrtfracVwn.
\end{align}
Moreover,
\begin{align*}
    \Vwn&=\supPiab\sumiM\fracwisqbarni\opsE_{\Zi\sim\tildecalDc}\regbra{\Delta^2(\Zi;\pia,\pib)} \\
    &\le\supPiab\max_{\source\in\sourceSet}\opsE_{\Zi\sim\tildecalDc}\regbra{\Delta^2(\Zi;\pia,\pib)}\cdot\sumiM\fracwisqbarni =\barV\skewness,
\end{align*}
where $\barV=\supPiab\max_{\source\in\sourceSet}\opsE_{\Zi\sim\tildecalDc}\bra{\Delta(\Zi;\pia,\pib)}$, which is a constant value. Note that the last equality holds by the skewness identity in established in Appendix \ref{app:Preliminaries-sourceDistributionSkewness}. Plugging this into inequality \eqref{eq:prop-WeightedRademacherBound-eq3}, we get
\begin{align*}
    \frakRwnF\le2A_1\sqrt{\frac{\barV\skewness}{n}}\le\bigO{\sqrtfracskewnessn}.
\end{align*}

\textit{\underline{Case 2}}: $A_2\sqrt{\swnC}>\frac{1}{2}\sqrt{\frakRwnF}$

In this case, one can easily rearrange terms to get that
\begin{align*}
    \frakRwnF<4A_2^2\swnC\le\littleo{\sqrtfracskewnessn}.
\end{align*}

\vspace{.5em}
Therefore, in either case, we have that at least $\frakRwnF\le\calO\Bigrbra{\sqrtfracskewnessn\,}$. We can plug this asymptotic bound into inequality \eqref{eq:prop-WeightedRademacherBound-eq1} to arrive at the desired result:
\begin{align*}
    &\frakRwnF \\
    &\le \rbra{14+6\kappa(\Pi)}\sqrtfracVwn + \sqrt{\frac{\Vwn}{n^2}} + \bigO{\sqrt{\bigO{\sqrtfracskewnessn}\swnC}} \\
    &\le \rbra{14+6\kappa(\Pi)}\sqrtfracVwn + \sqrt{\frac{\barV\skewness}{n^2}} + \bigO{\sqrt{\bigO{\sqrtfracskewnessn}\littleo{\sqrtfracskewnessn}}} \\
    &\le \rbra{14+6\kappa(\Pi)}\sqrtfracVwn + \littleo{\sqrtfracskewnessn} + \littleo{\sqrtfracskewnessn} \\
    &\le \rbra{14+6\kappa(\Pi)}\sqrtfracVwn + \littleo{\sqrtfracskewnessn}.
\end{align*}
\end{proof}

\subsection{Bounding Oracle Mixture Empirical Process}

\begin{proposition}\label{prop:OracleRegretBound}
Suppose Assumptions \ref{ass:dgp} and \ref{ass:LocalDataSizeScaling} hold.
Fix $\w\in\W$. For any $\delta\in(0,1)$, with probability at least $1-\delta$,
\begin{equation*}
    \supPiab|\Deltaw(\pia,\pib)-\Deltatildew(\pia,\pib)|\le C\kappa(\Pi)\sqrt{V\cdot\fracskewnessn\cdot\log\frac{1}{\delta}} + \littleop{\!\sqrtfracskewnessn},
\end{equation*}
where $C$ is a universal constant and
\begin{equation*}
    V=\maxC\supPi\opsE_{\Zi\sim\calDcec}\regbra{{\Gi(\pi(\Xi))^2}}.
\end{equation*}
\end{proposition}
\begin{proof}
    First, for each source $\source\in\sourceSet$, let $Z_1^c,\dots,Z_{\ni}^\source$ be $\nc$ independent random variables sampled from $\smash[t]{\tildecalDc}$, where each $\Zij=(\Xij,\smash[t]{\vecGij})\in\calZ=\calX\times\R^d$. Additionally, let $Z=\{\Zij\mid \source\in\sourceSet, i\in[n_c]\}$ represent the corresponding set of samples across all sources. 
    
    In Lemma \ref{lem:ExpectedOracleEqualsLocalPolicyValue}, we showed that $\E_{\Zi\sim\tildecalDc}\bra{Q(\Zi;\pi)}=\Qi(\pi)$.    
    This implies that
    \begin{align*}
        \E\regbra{\Qtilde_\w(\pi)}&=\sumijw\opsE\bra{Q(\Zij;\pi)}=\sumijw\Qi(\pi)=\sumiM\wi\Qi(\pi)=\Qw(\pi).
    \end{align*}
    Additionally,
    \begin{align*}
        \E\bigbra{\Deltatildew(\pia,\pib)}=\E\bigbra{\Qtildew(\pia)}-\E\bigbra{\Qtildew(\pib)}=\Qw(\pia)-\Qw(\pib)=\Deltaw(\pia,\pib).
    \end{align*}
    Therefore, we can follow a symmetrization argument to upper bound the expected oracle regret in terms of a Rademacher complexity, namely the weighted Rademacher complexity. Let $Z'$ be an independent copy of $Z$ and let $\eps=\{\epsij\mid \source\in\sourceSet,i\in[\nc]\}$ be a set of independent Rademacher random variables. Then,
    \begin{align*}
        &\E\bra{\supPiab|\Deltaw(\pia,\pib)-\Deltatildew(\pia,\pib)|} \\
        &=\E_Z\bra{\supPiab\biggabs{\E_{Z'}\biggbra{\sumijw \Delta(\Zijprime;\pia,\pib)}-\sumijw \Delta(\Zij;\pia,\pib)}} \\
        &=\E_Z\bra{\supPiab\biggabs{\E_{Z'}\biggbra{\sumijw \Delta(\Zijprime;\pia,\pib)-\sumijw \Delta(\Zij;\pia,\pib)}}} \\
        &\le\E_Z\bra{\E_{Z'}\bra{\supPiab\biggabs{\sumijw\Bigrbra{\Delta(\Zijprime;\pia,\pib)-\Delta(\Zij;\pia,\pib)}}}} \\
        &=\E_{Z,Z',\eps}\bra{\supPiab\biggabs{\sumijw\epsij\Bigrbra{\Delta(\Zijprime;\pia,\pib)-\Delta(\Zij;\pia,\pib)}}} \\
        &\le2\E_{Z,\eps}\bra{\supPiab\biggabs{\sumijw\epsij \Delta(\Zij;\pia,\pib)}} \\
        &=2\frakR_{\w,\barn}(\Delta\calF_\Pi) \\
        &\le 4\frakR_{\w,\barn}(\calF_\Pi).
    \end{align*}
    The first equalities and inequalities follow from standard symmetrization arguments, and the last inequality follows from Lemma \ref{lem:RademacherComplexityDifferenceBound}.
    Next, we use this bound on the expected oracle regret and Talagrand's inequality (Lemma \ref{lem:TalagrandInequality}), to establish a high-probability bound on the oracle regret.
    In particular, we identify the set of independent random variables $\tilde Z=\{\tildeZij=(\Xij,\fracwini\Gij)\mid \source\in\sourceSet, i\in[\nc]\}$ and the function class $\calH=\{h(\cdot;\pia,\pib)\mid\pia,\pib\in\Pi\}$ where
    \begin{align}
        h(\tildeZij;\pia,\pib)=\E\regbra{\Delta(\tildeZij;\pia,\pib)}-\Delta(\tildeZij;\pia,\pib),
    \end{align}
    which is uniformly bounded for any $\source\in\sourceSet$ and $i\in[\nc]$ by
    \begin{align*}
        \regabs{h(\tildeZij;\pia,\pib)}&=\fracwini\Bigabs{\E\bigbra{\Gij(\pia(\Xij))-\Gij(\pib(\Xij))}-\bigrbra{\Gij(\pia(\Xij))-\Gij(\pib(\Xij))}} \\
        &\le\fracwini 4U \\
        &\le4U\swnC\eqqcolon\UwnC,
    \end{align*}
    where $U>0$ is a uniform upper bound on $|\Gij(a)|$ for any $\source\in\sourceSet$ and $a\in\calA$ guaranteed by Assumption \ref{ass:dgp}, and where the last inequality follows from Inequality \eqref{eq:sourceDistributionRatioBound},
    Additionally, we have
    \begin{equation*}
        \swnC=\frac{1}{\sqrt{\min_{\source\in\sourceSet}n_c}}\sqrtfracskewnessn\le\littleo{\sqrtfracskewnessn},
    \end{equation*}
    as discussed in the proof of Proposition \ref{prop:WeightedRademacherBound}.
    Lastly, to use Talagrand's inequality, we set the constant $D$ (specified in Lemma \ref{lem:TalagrandInequality}) to be
    \begin{align*}
        D=\supPiab\sumij\E\bigbra{h^2(\tildeZij;\pia,\pib)} + 8\UwnC\E\bra{\supPiab\biggabs{\sumij\epsij h(\tildeZij;\pia,\pib)}}.
    \end{align*}
    By Lemma \ref{lem:ExpectedSupSumBound}, this choice of $D$ meets the required condition to use in Talagrand's inequality.
    In particular, we identify $\varphi:u\mapsto u^2$ defined over the set $\calU\subset\R$ containing all possible outputs of any function in $\calH$ given any realization of $\tildeZij$ for any $\source\in\sourceSet$ as input.
    The uniform bound established above on realizable outputs of $h$ given input $\tildeZij$ implies that $\calU\subset[-\UwnC,\UwnC]$, and therefore, the Lipschitz constant of $\varphi$ is $L=2\UwnC$, as required.

    Next, after setting $t$ to be the positive solution of
    \begin{align*}
        \frac{t^2}{CD+C\UwnC t}=\log(C/\delta),
    \end{align*}
    Talagrand's inequality guarantees
    \begin{align*}
        &\P\rbra{\Bigabs{\supPiab\bigabs{\Deltaw(\pia,\pib)-\Deltatildew(\pia,\pib)}-\E\Bigbra{\supPiab\bigabs{\Deltaw(\pia,\pib)-\Deltatildew(\pia,\pib)}}}\ge t} \\
        &=\P\rbra{\biggabs{\supPi\Bigabs{\sumij h(\tildeZij;\pia,\pib)} - \E\bra{\supPi\biggabs{\sumij h(\tildeZij;\pia,\pib)} }} \ge t} \\
        &\le C\exp\rbra{-\frac{t}{C\UwnC}\log\rbra{1+\frac{\UwnC t}{D}}} \\
        &\le C\exp\rbra{-\frac{t^2}{CD+C\UwnC t}}=\delta.
    \end{align*}
    Here, we used the inequality $\log(1+x)\ge\frac{x}{1+x}$ for any $x\ge0$.
    Observe that, by construction,
    \begin{align*}
        t&=\frac{1}{2}C\UwnC\log(C/\delta)+\sqrt{\frac{1}{4}C^2\UwnC^2\log^2(C/\delta)+CD\log(C/\delta)} \\
        &\le C\UwnC\log(C/\delta) + \sqrt{CD\log(C/\delta)}
    \end{align*}
    and
    \begin{align*}
        D&=\supPiab\sumij\E\bigbra{h^2(\tildeZij;\pia,\pib)} + 8\Uwn\E\bra{\supPiab\biggabs{\sumij\epsij h(\tildeZij;\pia,\pib)}} \\
        &=\supPiab\sumij\fracwinisq\E\bra{\bigrbra{\E\bra{\Delta(\Zij;\pia,\pib)}-\Delta(\Zij;\pia,\pib)}^2} \\
        &\hspace{1em} + 8\UwnC\E\bra{\supPiab\Biggabs{\sumij\fracwini\epsij\bigrbra{\E\bra{\Delta(\Zij;\pia,\pib)}-\Delta(\Zij;\pia,\pib)}}} \\
        &=\supPiab\sumij\fracwinisq\Bigrbra{\E\bra{\Delta^2(\Zij;\pia,\pib)}-\E\bra{\Delta(\Zij;\pia,\pib)}^2} \\
        &\hspace{1em} + 8\UwnC\E\bra{\supPiab\Biggabs{\sumij\fracwini\epsij\bigrbra{\E\bra{\Delta(\Zij;\pia,\pib)}-\Delta(\Zij;\pia,\pib)}}} \\
        &\le \supPiab\sumij\fracwinisq\E\bra{\Delta^2(\Zij;\pia,\pib)} + 16\UwnC\E\bra{\supPiab\Biggabs{\sumij\fracwini\epsij\Delta(\Zij;\pia,\pib)}} \\
        &\le\supPiab\sumiM\fracwisqni\E\bra{\Delta^2(\Zij;\pia,\pib)} + 16\UwnC\frakRwn(\Delta\calF_\Pi) \\
        &\le\supPiab\sumiM\fracwisqni\E\bra{\Delta^2(\Zij;\pia,\pib)} + 32\UwnC\frakRwnF \\
        &=\fracVwn + 128U\frakRwnF\swnC.
    \end{align*}

    Therefore, with this setup, Talagrand's inequality guarantees that with probability at least $1-\delta$
    \begin{align*}
        &\supPiab\bigabs{\Deltaw(\pia,\pib)-\Deltatildew(\pia,\pib)} \\
        &\le\E\bra{\supPiab\bigabs{\Deltaw(\pia,\pib)-\Deltatildew(\pia,\pib)}} + t \\
        &=4\frakRwnF + \sqrt{CD\log(C/\delta)} + C\UwnC\log(C/\delta) \\
        &\le4\frakRwnF + \sqrt{C\rbra{\fracVwn + 128U\frakRwnF\swnC}\log\rbra{C/\delta}} + 4CU\swnC\log\rbra{C/\delta} \\
        &\le4\frakRwnF + \sqrt{C\log(C/\delta)\fracVwn } + \sqrt{128UC\log(C/\delta) \frakRwnF\swnC} + 4UC\log(C/\delta)\swnC\\
        &\le \rbra{\rbra{56+24\kappa(\Pi)}\sqrtfracVwn + \littleo{\sqrtfracskewnessn}} + \sqrt{C\log(C/\delta)\fracVwn} \\
        &\hspace{1.5em}+ \sqrt{\bigOp{\sqrtfracskewnessn}\littleo{\sqrtfracskewnessn}} + \littleop{\sqrtfracskewnessn} \\
        &\le\rbra{56+24\kappa(\Pi) + \sqrt{C\log(C/\delta)}}\sqrtfracVwn + \littleop{\sqrtfracskewnessn}.
    \end{align*}
    It holds that the constant in Talagrand's inequality satisfies $C\ge 1$.
    Moreover, without significant loss of generality, we can consider $\kappa(\Pi)\ge 1$. Lastly, we can map $\delta/C\mapsto\delta$ to get
    \begin{align*}
        \supPiab\bigabs{\Deltaw(\pia,\pib)-\Deltatildew(\pia,\pib)}\le C'\kappa(\Pi)\sqrt{\fracVwn\log\rbra{\frac{1}{\delta}}} + \littleop{\sqrtfracskewnessn}
    \end{align*}
    where we set $C'=56 + 25 + \sqrt{C}$.
    Here, we used the bounds previously established in the proof of Proposition \ref{prop:WeightedRademacherBound} that $\frakRwnF\le\calO\bigrbra{\sqrt{\skewness/n}\,}$ and $\swnC\le o\bigrbra{\sqrt{\skewness/n}\,}$.

    Lastly, we decompose the weighted variance term $\Vwn$ as
    \begin{align*}
        \Vwn&=\supPiab\sumiM\fracwisqbarni\opsE_{\Zi\sim\tildecalDc}\bra{\Delta^2(\Zi;\pia,\pib)} \\
        &=\supPiab\sumiM\fracwisqbarni\opsE_{\Zi\sim\calDcec}\bigbra{\bigrbra{\Gi(\pia(\Xi))-\Gi(\pib(\Xi))}^2} \\
        &\le\maxC\supPiab\opsE_{\Zi\sim\calDcec}\bigbra{\bigrbra{\Gi(\pia(\Xi))-\Gi(\pib(\Xi))}^2}\cdot\sumiM\fracwisqbarni \\
        &\le4\cdot\maxC\supPi\opsE_{\Zi\sim\calDcec}\regbra{{\Gi(\pi(\Xi))^2}}\cdot\sumiM\fracwisqbarni \\
        &=4V\cdot\skewness.
    \end{align*}
    We absorb the factor of $\sqrt{4}$ into the universal constant to get the desired result.
    
\end{proof}

\subsection{Bounding Approximate Mixture Empirical Process}\label{app:BoundingApproximateRegret}

First, we state a more general form of Assumption~\ref{ass:FiniteSampleError} that is sufficient for bounding the approximate mixture empirical process.

\begin{assumption}
    For any source $\s\in\sourceSet$, the source estimates $\muhati$ and $\ohati$ of the nuisance parameters $\mui$ and $\oi$, respectively, trained on $\nc$ source data points satisfy
    the following squared error bounds:
    \begin{gather*}
        \E_{\calDc}\!\bigbra{\norm{\muhati(\Xi)-\mui(\Xi)}_2^2}\le\frac{o(1)}{\nc^{\zeta_\mu}}, \\
        \E_{\calDc}\!\bigbra{\norm{\ohati(\Xi)-\oi(\Xi)}_2^2}\le\frac{o(1)}{\nc^{\zetao}},
    \end{gather*}
    for some $0<\zeta_\mu,\zetao<1$ with $\zeta_\mu+\zetao\ge 1$.
\end{assumption}

This assumption subsumes Assumption~\ref{ass:FiniteSampleError} and allows flexibility in the estimation of nuisance parameters. It allows less accurate estimation of the inverse conditional propensities if we can estimate the response response functions at a faster rate, and vice versa. Now, we can return to bounding the approximate mixture empirical process with this more general assumption on the error rates of the nuisance parameter estimates.

\begin{proposition}\label{prop:ApproximateRegretBound}
Suppose Assumptions \ref{ass:dgp}, \ref{ass:LocalDataSizeScaling}, and \ref{ass:FiniteSampleError} hold. Fix $\w\in\W$. Then,
\begin{equation*}
    \supPiab|\Deltatildew(\pia,\pib)-\Deltahatw(\pia,\pib)|\le \littleop{\!\sqrtfracskewnessn}.
\end{equation*}
\end{proposition}
\begin{proof}
    Recall that $\{(\Xij,\Aij,\Yij)\}_{i=1}^{\nc}$ is the data collected by source $\source\in\sourceSet$ as described in Section \ref{sec:Preliminaries-Data}.
    We assume each source estimates the local nuisance parameters using a cross-fitting strategy, as discussed in Algorithm \ref{alg:CAIPW}. Under this strategy, each source $\source\in\sourceSet$ divides their local dataset into $K$ folds, and for each fold $k$, the source estimates $\mui$ and $\oi$ using the rest $K-1$ folds. Let $k_c:[\nc]\to[K]$ denote the surjective mapping that maps a data point index to its corresponding fold containing the data point. We let $\muhatikij$ and $\ohatikij$ denote the estimators of $\mui$ and $\oi$ fitted on the $K-1$ folds of source $\source$ other than $\kij$.

    As discussed Section \ref{sec:Approach}, recall the oracle AIPW scores
    \begin{equation*}
        \Gij(a)=\mui(\Xij;a)+\bigrbra{\Yij(a)-\mui(\Xij;a)}\oi(\Xij;a)\ones\{\Aij=a\}
    \end{equation*}
    and approximate AIPW scores
    \begin{equation*}
        \Ghatij(a)=\muhatikij(\Xij;a)+\bigrbra{\Yij-\muhatikij(\Xij;a)}\ohatikij(\Xij;a)\ones\{\Aij=a\}
    \end{equation*}
    for any $a\in\calA$, where $\kij$ is the fold corresponding to data point $i$ of source $\source$. One can verify that the difference between the oracle and approximate AIPW scores can be expressed as
    \begin{align*}
        \Ghatij(a)-\Gij(a)=\Gijp(a)+\Gijpp(a)+\Gijppp(a),
    \end{align*}
    where
    \begin{align*}
        \Gijp(a)&=\rbra{\muhatikij(\Xij;a)-\mui(\Xij;a)}\bigrbra{1-\oi(\Xij;a)\ones\{\Aij=a\}}, \\
        \Gijpp(a)&=\bigrbra{\Yij(a)-\mui(\Xij;a)}\rbra{\ohatikij(\Xij;a)-\oi(\Xij;a)}\ones\{\Aij=a\}, \\
        \Gijppp(a)&=\rbra{\mui(\Xij;a)-\muhatikij(\Xij;a)}\rbra{\ohatikij(\Xij;a)-\oi(\Xij;a)}\ones\{\Aij=a\}.
    \end{align*}

    This induces the following decomposition of the approximate regret:
    \begin{align*}
        \Deltahatw(\pia,\pib)-\Deltatildew(\pia,\pib)=S_1(\pia,\pib)+S_2(\pia,\pib)+S_3(\pia,\pib),
    \end{align*}
    where
    \begin{align*}
        S_1(\pia,\pib)&=\sumijw\Gijp(\pia(\Xij))-\Gijp(\pib(\Xij)) ,\\
        S_2(\pia,\pib)&=\sumijw\Gijpp(\pia(\Xij))-\Gijpp(\pib(\Xij)), \\
        S_3(\pia,\pib)&=\sumijw\Gijppp(\pia(\Xij))-\Gijppp(\pib(\Xij)).
    \end{align*}
    We further decompose $S_1$ and $S_2$ by folds as follows:
    \begin{align*}
        S_1(\pia,\pib)&=\sumkK S_1^k(\pia,\pib), \\
        S_2(\pia,\pib)&=\sumkK S_2^k(\pia,\pib),
    \end{align*}
    where
    \begin{align*}
        S_1^k(\pia,\pib)=\sumiM\fracwini\sum_{\{i|\kij=k\}}\Gijp(\pia(\Xij))-\Gijp(\pib(\Xij)), \\
        S_2^k(\pia,\pib)=\sumiM\fracwini\sum_{\{i|\kij=k\}}\Gijpp(\pia(\Xij))-\Gijpp(\pib(\Xij)),
    \end{align*}
    for each $k\in[K]$. To determine a bound on the approximate regret, we will establish high probability bounds for the worst-case absolute value over policies of each term in this decomposition. For convenience, for any policy $\pi$, we will denote $\pi(x;a)=\ones\{\pi(x)=a\}$.

    \paragraph{Bounding $S_1$.}
    We wish to bound the quantity $\supPiab\abs{S_1(\pia,\pib)}$. We first bound the quantity $\supPiab\abs{S_1^k(\pia,\pib)}$ for any $k\in[K]$.
    
    First, note that since $\muhatikij$ is estimated using data outside fold $\kij$, when we condition on the data outside fold $\kij$, $\muhatikij$ is fixed and each term in $S_1(\pia,\pib)$ is independent. This allows us to compute
    \begin{align*}
        &\E\bra{\Gijp(\pia(\Xij))-\Gij(\pib(\Xij))} \\
        &=\sumaA\E\bra{\bigrbra{\pia(\Xij;a)-\pib(\Xij;a)}\rbra{\muhatikij(\Xij;a)-\mui(\Xij;a)}\bigrbra{1-\oi(\Xij;a)\ones\{\Aij=a\}}} \\
        &=\sumaA\E\bra{\E\bra{\bigrbra{\pia(\Xij;a)-\pib(\Xij;a)}\rbra{\muhatikij(\Xij;a)-\mui(\Xij;a)}\bigrbra{1-\oi(\Xij;a)\ones\{\Aij=a\}}\Bigmid \Xij}} \\
        &=\sumaA\E\bra{\bigrbra{\pia(\Xij;a)-\pib(\Xij;a)}\rbra{\muhatikij(\Xij;a)-\mui(\Xij;a)}\E\bra{1-\oi(\Xij;a)\ones\{\Aij=a\}\Bigmid \Xij}}=0
    \end{align*}
    Therefore,
    \begin{align*}
        &K\supPiab\abs{S_1^k(\pia,\pib)} \\
        &\le\supPiab\Biggabs{\sumiM\frac{\wi}{\nc/K}\sum_{\{i|\kij=k\}}\Gijp(\pia(\Xij))-\Gijp(\pib(\Xij))} \\
        &=\supPiab\Biggabs{\sumiM\frac{\wi}{\nc/K}\sum_{\{i|\kij=k\}}\bigrbra{\Gijp(\pia(\Xij))-\Gijp(\pib(\Xij))}-\E\bra{\Gijp(\pia(\Xij))-\Gijp(\pib(\Xij))}}.
    \end{align*}
    Identifying $\Gijp$ with $\Gij$ and sample sizes $\nC/K$ with $\nC$, the right-hand side in the above inequality is effectively an oracle regret and so we can apply Proposition \ref{prop:OracleRegretBound} to obtain that with probability at least $1-\delta$,
    \begin{align*}
        &K\supPiab\abs{S_1^k(\pia,\pib)} \\
        &\le\supPiab\biggabs{\sumiM\frac{\wi}{\nc/K}\sum_{\{j|k_i(j)=k\}}\bigrbra{\Gijp(\pia(\Xij))-\Gijp(\pib(\Xij))}-\E\bra{\Gijp(\pia(\Xij))-\Gijp(\pib(\Xij))}} \\
        &\le C_{\Pi,\delta}\sqrt{\maxC\supPi\E\bra{\Gijp(\pi(\Xij))^2\mid\muhatikij}\cdot\frac{\skewness}{n/K}} + \littleop{\sqrt{\frac{\skewness}{n/K}}} \\
        &\le C_{\Pi,\delta}\rbra{1/\eta-1}\sqrt{K\maxC\E\bra{||\muhatikij(\Xij)-\mui(\Xij)||_2^2\mid\muhatikij}\cdot\fracskewnessn}+ \littleop{\sqrtfracskewnessn},
    \end{align*}
    where $C_{\Pi,\delta}=C\kappa(\Pi)\sqrt{\log(1/\delta)}$ for some universal constant $C$, and $\eta=\min_{\source\in\sourceSet}\eta_\source$ for $\eta_\source$ in the overlap assumption stated in in Assumption \ref{ass:dgp}. The last inequality follows from a uniform bound on $\Gijp(\pi(\Xij))$ and the overlap assumption.
    
    By the assumption on finite sample error bounds for the nuisance functions stated above, for every $\source\in\sourceSet$
    \begin{align*}
        \E\bra{||\muhatikij(\Xij)-\mui(\Xij)||^2\mid\muhatikij}\le\frac{g_\source\rbra{\alpha_K\nc}}{\rbra{\alpha_K\nc}^{\zeta_\mu}},
    \end{align*}
    where $\alpha_K=1-K^{-1}$, $g_\source$ is some decreasing function, and $0<\zeta_\mu<1$. Then,
    \begin{align*}
        \maxC\E\bra{||\muhatikij(\Xij)-\mui(\Xij)||^2\mid\muhatikij}&\le\maxC\frac{g_\source(\alphaK\nc)}{(\alphaK\nc)^\zetamu} \\
        &\le\frac{\max_{\source\in\sourceSet} g_\source(
        \alphaK\nc)}{\alphaK^\zetamu\cdot\min_{\source\in\sourceSet}\nc^\zetamu} \\
        &\le\frac{\max_{\source\in\sourceSet} g_\source(
        \alphaK\nc)}{\alphaK^\zetamu\cdot\min_{\source\in\sourceSet}\nc^\zetamu}.
    \end{align*}
    By the local data size scaling assumption in Assumption \ref{ass:LocalDataSizeScaling}, for any $\source\in\sourceSet$, we have that $\nc=\Omega(\nu_c(n))$ where $\nu_c$ is an increasing function.
    In other words, there exists a constant $\tau>0$ such that $\nc\ge \tau\nuc(n)$ for sufficiently large $n$. Then, since $\gc$ is decreasing, $\gc(\alphaK n_c)< \gc(\tau\alphaK\nuc(n))$ for sufficiently large $n$.
    Moreover, since $\nuc$ is increasing and $\tau\alpha_K>0$, $\tilde\nu_c=\tau\alphaK\nuc$ is also increasing, and since $\gc$ is decreasing, the composition $\tilde g_\source=\gc\circ\tilde\nu_c$ is decreasing. Therefore, $\gc(\alphaK\nc)$ is asymptotically bounded by a decreasing function $\tilde g_\source$ of $n$.
    This observation and the fact that the maximum of a set of decreasing functions is itself decreasing imply that $\max_{\source\in\sourceSet}\gc(\alphaK\nc)$ is asymptotically bounded by the decreasing function $\tilde g$ defined by $\tilde g(n)=\max_{\source\in\sourceSet}\tilde g_\source(n)$.
    In other words,
    \begin{align*}
        \max_{\source\in\sourceSet}\gc(\alphaK\nc)\le\tilde g(n)\le o(1).
    \end{align*}
    Additionally, since $\nc=\Omega(\nuc(n))$ and $\zetamu>0$, we also have that
    \begin{align*}
        \frac{1}{\min_{\source\in\sourceSet}\nc^\zetamu}\le o(1).
    \end{align*}
    These two observations imply
    \begin{align*}
        \maxC\E\bra{||\muhatikij(\Xij)-\mui(\Xij)||^2\mid\muhatikij} &\le\frac{\max_{\source\in\sourceSet} g_\source(
        \alphaK\nc)}{\alphaK^\zetamu\cdot\min_{\source\in\sourceSet}\nc^\zetamu}\le \littleo{1}.
    \end{align*}
    Therefore,
    \begin{align*}
        &\supPiab\abs{S_1^k(\pia,\pib)} \\
        &\le C_{\Pi,\delta}\rbra{1/\eta-1}\sqrt{\frac{1}{K}\maxC\E\bra{||\muhatikij(\Xij)-\mui(\Xij)||_2^2\mid\muhatikij}\cdot\fracskewnessn}+ \littleop{\sqrtfracskewnessn} \\
        &\le C_{\Pi,\delta}\rbra{1/\eta-1}\sqrt{\frac{1}{K}\cdot o\rbra{\fracskewnessn}}+ \littleop{\sqrtfracskewnessn} \\
        &\le\littleop{\sqrtfracskewnessn},
    \end{align*}
    and
    \begin{align*}
        \supPiab\abs{S_1(\pia,\pib)}&\le\sumkK\supPiab\abs{S_1^k(\pia,\pib)}\le\littleop{\sqrtfracskewnessn}.
    \end{align*}

    \paragraph{Bounding $S_2$.}
    The bound for $\supPiab\abs{S_2(\pia,\pib)}$ follows the same argument as that of $S_1$. We first bound $\supPiab\abs{S_2^k(\pia,\pib)}$ for any $k\in[K]$.
    
    First, note that since $\ohatikij$ is estimated using data outside fold $\kij$, when we condition on the data outside fold $\kij$, $\ohatikij$ is fixed and each term in $S_2(\pia,\pib)$ is independent. This allows us to compute
    \begin{align*}
        &\E\bra{\Gijpp(\pia(\Xij))-\Gijpp(\pib(\Xij))} \\
        &=\E\bra{\sumaA\rbra{\pia(\Xij;a)-\pib(\Xij;a)}\rbra{\Yij(a)-\mui(\Xij;a)}\rbra{\ohatikij(\Xij;a)-\oi(\Xij;a)}\ones\{\Aij=a\}} \\
        &=\E\bra{\rbra{\pia(\Xij;\Aij)-\pib(\Xij;\Aij)}\rbra{\Yij(\Aij)-\mui(\Xij;\Aij)}\rbra{\ohatikij(\Xij;a)-\oi(\Xij;a)}} \\
        &=\E\bra{\E\bra{\rbra{\pia(\Xij;\Aij)-\pib(\Xij;\Aij)}\rbra{\Yij(\Aij)-\mui(\Xij;\Aij)}\rbra{\ohatikij(\Xij;a)-\oi(\Xij;a)}\Bigmid \Xij,\Aij}} \\
        &=\E\bra{\rbra{\pia(\Xij;\Aij)-\pib(\Xij;\Aij)}\E\Bigbra{\Yij(\Aij)-\mui(\Xij;\Aij)\mid\Xij,\Aij}\rbra{\ohatikij(\Xij;a)-\oi(\Xij;a)}}=0
    \end{align*}
    Therefore, we can follow the exact same argument as above, eliciting Proposition \ref{prop:OracleRegretBound}, to obtain that with probability at least $1-\delta$,
    \begin{align*}
        &K\supPiab\abs{S_2^k(\pia,\pib)} \\
        &\le\supPiab\Biggabs{\sumiM\frac{\wi}{\nc/K}\sum_{\{i|\kij=k\}}\Gijpp(\pia(\Xij))-\Gijpp(\pib(\Xij))} \\
        &=\supPiab\Biggabs{\sumiM\frac{\wi}{\nc/K}\sum_{\{i|\kij=k\}}\bigrbra{\Gijpp(\pia(\Xij))-\Gijpp(\pib(\Xij))}-\E\bra{\Gijpp(\pia(\Xij))-\Gijpp(\pib(\Xij))}} \\
        &\le\supPiab\biggabs{\sumiM\frac{\wi}{\nc/K}\sum_{\{j|k_i(j)=k\}}\bigrbra{\Gijpp(\pia(\Xij))-\Gijpp(\pib(\Xij))}-\E\bra{\Gijpp(\pia(\Xij))-\Gijpp(\pib(\Xij))}} \\
        &\le C_{\Pi,\delta}\sqrt{\maxC\E\bra{\Gijpp(\pi(\Xij))^2\mid\ohatikij}\cdot\frac{\skewness}{n/K}} + \littleop{\sqrt{\frac{\skewness}{n/K}}} \\
        &\le C_{\Pi,\delta}\sqrt{2BK\maxC\E\bra{||\ohatikij(\Xij)-\oi(\Xij)||_2^2\mid\ohatikij}\cdot\fracskewnessn}+ \littleop{\sqrtfracskewnessn},
    \end{align*}
    where $C_{\Pi,\delta}=C\kappa(\Pi) \sqrt{\log(1/\delta)}$ for some universal constant $C$, and $B=\maxC B_\source$ for the bounds $B_\source$ on the outcomes defined in Assumption \ref{ass:dgp}.
    The last inequality follows from a uniform bound on $\Gijpp(\pi(\Xij))$.

    We follow the exact same argument as above to get
    \begin{align*}
        \maxC\E\bra{||\ohatikij(\Xij)-\oi(\Xij)||^2\mid\ohatikij} &\le o\rbra{\fracskewnessn}.
    \end{align*}

    Therefore,
    \begin{align*}
        &\supPiab\abs{S_2^k(\pia,\pib)} \\
        &\le C_{\Pi,\delta}\sqrt{\frac{2B}{K}\maxC\E\bra{||\ohatikij(\Xij)-\oi(\Xij)||_2^2\mid\ohatikij}}+ \littleop{\sqrtfracskewnessn} \\
        &\le C_{\Pi,\delta}\sqrt{\frac{2B}{K}\cdot o\rbra{\fracskewnessn}}+ \littleop{\sqrtfracskewnessn} \\
        &\le\littleop{\sqrtfracskewnessn},
    \end{align*}
    and
    \begin{align*}
        \supPiab\abs{S_2(\pia,\pib)}&\le\sumkK\supPiab\abs{S_2^k(\pia,\pib)}\le\littleop{\sqrtfracskewnessn}.
    \end{align*}

    \paragraph{Bounding $S_3$.}
    Next, we bound the contribution from $S_3$. We have that
    \begin{align*}
        &\supPiab\abs{S_3(\pia,\pib)} \\
        &=\supPiab\abs{\sumijw\Gijppp(\pia(\Xij))-\Gijppp(\pib(\Xij))} \\
        &\le2\abs{\sumijw\sum_{a\in\calA}\rbra{\mui(\Xij;a)-\muhatikij(\Xij;a)}\rbra{\ohatikij(\Xij;a)-\oi(\Xij;a)}} \\
        &\le2\sqrt{\sumijw\bignorm{\mui(\Xij)-\muhatikij(\Xij)}_2^2}\sqrt{\sumijw\bignorm{\ohatikij(\Xij)-\oi(\Xij)}_2^2} \\
        &\le2\sqrt{\sumiM\wi\frac{\gc(\alphaK\nc)}{(\alphaK\nc)^\zetamu}}\sqrt{\sumiM\wi\frac{\gc(\alphaK\nc)}{(\alphaK\nc)^\zetao}} \\
        &\le\frac{2}{\alphaK^{(\zetamu+\zetao)/2}}\sqrt{\maxC\frac{\wi}{\nc^\zetamu}\sumiM\gc(\alphaK\nc)}\sqrt{\maxC\frac{\wi}{\nc^\zetao}\sumiM\gc(\alphaK\nc)} \\
        &=\frac{2}{\alphaK^{(\zetamu+\zetao)/2}}\sumiM\gc(\alphaK\nc)\sqrt{\maxC\frac{\wi^2}{\nc^{\zetamu+\zetao}}} \\
        &\le\frac{2}{\alphaK^{(\zetamu+\zetao)/2}}\sumiM\gc(\alphaK\nc)\sqrt{\maxC\frac{\wi^2}{\nc}} \\
        &\le\frac{2}{\alphaK^{(\zetamu+\zetao)/2}}\sumiM\gc(\alphaK\nc)\sqrt{\sumiM\frac{\wi^2}{\nc}} \\
        &\le\frac{2}{\alphaK^{(\zetamu+\zetao)/2}}\sumiM\gc(\alphaK\nc)\sqrtfracskewnessn.
    \end{align*}
    As discussed earlier, $\gc(\alphaK\nc)$ is asymptotically bounded by a decreasing function of $n$. Since the sum of decreasing functions is decreasing, $\sumiM\gc(\alphaK\nc)$ is asymptotically bounded by a decreasing function $\tilde g$ in $n$. In other words, $\sumiM\gc(\alphaK\nc)\le\tilde g(n)\le o(1)$.
    Therefore,
    \begin{align*}
        \supPiab\abs{S_3(\pia,\pib)}\le\frac{2}{\alphaK^{(\zetamu+\zetao)/2}}\cdot o(1)\cdot\sqrtfracskewnessn\le\littleo{\sqrtfracskewnessn}.
    \end{align*}

    \paragraph{Combine results.}
    Putting all the above bounds together, we have
    \begin{align*}
        \supPiab|\Deltatildew(\pia,\pib)-\Deltahatw(\pia,\pib)|&\le\supPiab\abs{S_1(\pia,\pib)+S_2(\pia,\pib)+S_3(\pia,\pib)} \\
        &\le\supPiab\abs{S_1(\pia,\pib)}+\supPiab\abs{S_2(\pia,\pib)}+\supPiab\abs{S_3(\pia,\pib)} \\
        &\le\littleop{\sqrtfracskewnessn}.
    \end{align*}
\end{proof}

\subsection{Bounding Mixture Empirical Process}

\begin{proposition}\label{prop:BoundingMixtureEmpiricalProcess}
    Suppose Assumptions \ref{ass:dgp}, \ref{ass:LocalDataSizeScaling}, and \ref{ass:FiniteSampleError} hold.
    For any $\epsilon>0$, let $\W_\epsilon$ denote the minimal $\epsilon$-covering set of $\W$ under the $\ell_1$ distance. For any $\epsilon>0$, any $\delta\in(0,1)$, and any $\w\in\W_\epsilon$, with probability at least $1-\delta$,
    \begin{equation*}
        \supPiab\regabs{\Deltaw(\pia,\pib)-\Deltahatw(\pia,\pib)}\le \rateW,
    \end{equation*}
    where
    \begin{align*}
        \rateW\coloneqq C\kappa(\Pi)\sqrt{V\cdot\fracskewnessWn\cdot\log\frac{\abs{\W_\epsilon}}{\delta}} +  \littleop{\!\sqrtfracskewnessWn},
    \end{align*}
    with $C$ being a universal constant and
    \begin{equation*}
        V=\maxC\supPi\E_{\calDcec}\!\regbra{\Gi(\pi(\Xi))^2}.
    \end{equation*}
\end{proposition}
\begin{proof}
    By Propositions \ref{prop:OracleRegretBound} and \ref{prop:ApproximateRegretBound}, for any fixed choice of $\w\in\We$, with probability at least $1-\delta$,
    \begin{align*}
        &\supPiab\regabs{\Deltaw(\pia,\pib)-\Deltahatw(\pia,\pib)} \\
        &\le \supPiab\regabs{\Deltaw(\pia,\pib)-\Deltatildew(\pia,\pib)} + \supPiab\regabs{\Deltatildew(\pia,\pib)-\Deltahatw(\pia,\pib)} \\
        &\le C\kappa(\Pi)\sqrt{V\cdot\fracskewnessn\cdot\log\frac{1}{\delta}} + \littleop{\!\sqrtfracskewnessn},
    \end{align*}
    where $C$ is a universal constant.
    In other words, for any fixed $\w\in\We$,
    \begin{align*}
        \P\rbra{\supPiab\regabs{\Deltaw(\pia,\pib)-\Deltahatw(\pia,\pib)}>C\kappa(\Pi)\sqrt{V\cdot\fracskewnessn\cdot\log\frac{1}{\delta}} + \littleop{\!\sqrtfracskewnessn}}\le\delta.
    \end{align*}
    Then, by a union bound over $\We$, it follows that
    \begin{align*}
        &\P\rbra{\forall\w\in\We, \supPiab\regabs{\Deltaw(\pia,\pib)-\Deltahatw(\pia,\pib)}>C\kappa(\Pi)\sqrt{V\cdot\fracskewnessn\cdot\log\frac{1}{\delta}} + \littleop{\!\sqrtfracskewnessn}} \\
        &\le \sum_{\w\in\We}\P\rbra{ \supPiab\regabs{\Deltaw(\pia,\pib)-\Deltahatw(\pia,\pib)}>C\kappa(\Pi)\sqrt{V\cdot\fracskewnessn\cdot\log\frac{1}{\delta}} + \littleop{\!\sqrtfracskewnessn}} \\
        &\le\regabs{\We}\delta.
    \end{align*}
    Mapping $\delta\mapsto\delta/\regabs{\We}$, this implies that for every $\w\in\We$, with probability at least $1-\delta$,
    \begin{align*}
        \supPiab\regabs{\Deltaw(\pia,\pib)-\Deltahatw(\pia,\pib)}&\le C\kappa(\Pi)\sqrt{V\cdot\fracskewnessn\cdot\log\frac{\regabs{\We}}{\delta}} + \littleop{\!\sqrtfracskewnessn} \\
        &\le C\kappa(\Pi)\sqrt{V\cdot\fracskewnessWn\cdot\log\frac{\regabs{\We}}{\delta}} + \littleop{\!\sqrtfracskewnessWn}
    \end{align*}

\end{proof}

\subsection{EG-OPO Suboptimality Bound}
\label{app:EGOPO-Bound}

\SuboptimalityBoundLemma*

\begin{proof}
    The regret guarantee of the exponentiated gradient algorithm, specifically Corollary 2.14 of  \citep{shalev2012online}, states that for any $\wp\in\We$
    \begin{align*}
        -\sum_{t=1}^T\opsE_{\w\sim\rho_t}\bra{\frac{\Rw(\pi_t)}{\Bhat}} \le -\sum_{t=1}^T\frac{\Rhatwp(\pi_t)}{\Bhat} + \frac{\log\regabs{\We}}{\eta\Bhat} + \eta\Bhat T.
    \end{align*}
    Substituting $\eta\Bhat=\sqrt{\log\regabs{\We}/T}$ gives
    \begin{align*}
        -\sum_{t=1}^T\opsE_{\w\sim\rho_t}\bra{\frac{\Rw(\pi)}{\Bhat}} \le -\sum_{t=1}^T\frac{\Rhatwp(\pi)}{\Bhat} + 2\sqrt{T\log\regabs{\We}}.
    \end{align*}

    Recalling that $P_t=\Uniform(\pi_1,\dots,\pi_t)$, following the proof techniques of \cite{freund1996game} gives that
    \begin{align*}
        \opsE_{\pi\sim P_T}\bra{\frac{\Rhatwp(\pi)}{\Bhat}}&=\frac{1}{T}\sum_{t=1}^T\frac{\Rhatwp(\pi_t)}{\Bhat} \\
        &\le\frac{1}{T}\sum_{t=1}^T \opsE_{\w\sim\rho_t}\bra{\frac{\Rw(\pi_t)}{\Bhat}} + 2\sqrt{\frac{\log\regabs{\We}}{T}} \\
        &=\frac{1}{T}\sum_{t=1}^T \minPi\opsE_{\w\sim\rho_t}\bra{\frac{\hat{R}_{\w}(\pi)}{\Bhat}} + 2\sqrt{\frac{\log\regabs{\We}}{T}} \\
        &\le\minPi\frac{1}{T}\sum_{t=1}^T\opsE_{\w\sim\rho_t}\bra{\frac{\Rw(\pi)}{\Bhat}} + 2\sqrt{\frac{\log\regabs{\We}}{T}} \\
        &\le\minPi\max_{\w\in\W}\frac{\Rhatw(\pi)}{\Bhat} + 2\sqrt{\frac{\log\regabs{\We}}{T}}.
    \end{align*}

    Multiplying through by $\Bhat$ gives
    \begin{align*}
        \opsE_{\pi\sim P_T}\bra{\Rhatwp(\pi)}&\le\minPi\max_{\w\in\W}\Rhatw(\pi) + 2\Bhat\sqrt{\frac{\log\regabs{\We}}{T}}.
    \end{align*}
\end{proof}

\begin{proposition}\label{prop:HedgeOPOSuboptimalityBound}
    For any $T$ and any $\wp\in\We$,
    \begin{equation*}
        \opsE_{\pi\sim P_T}\bra{\Rhatwp(\pi)}\le \minPi\max_{\w\in\W}\Rw(\pi) + \rateW + 2\hat B\sqrt{\frac{\log\regabs{\We}}{T}}
    \end{equation*}
    where $\hat B$ is a uniform bound on $\Rhatwp(\pi)$.
\end{proposition}
\begin{proof}
    By Proposition \ref{prop:BoundingMixtureEmpiricalProcess},
    \begin{align*}
        \Rhatwp(\pi)&=\Deltahatwp(\pihatwp,\pi) \\
        &\le\Deltawp(\pihatwp,\pi)+\rateW \\
        &\le\Deltawp(\pistarwp,\pi)+\rateW \\
        &=\Rwp(\pi)+\rateW.
    \end{align*}
    Therefore, using the result of \ref{lem:EG-OPO-SuboptimalityBound}, it follows that for any $\wp\in\We$
    \begin{align*}
        \opsE_{\pi\sim P_T}\bra{\Rhatwp(\pi)}&\le\minPi\maxwW\Rhatw(\pi) + 2\Bhat\sqrt{\frac{\log\regabs{\We}}{T}} \\
        &\le\minPi\maxwW\Rw(\pi) + \rateW + 2\Bhat\sqrt{\frac{\log\regabs{\We}}{T}}.
    \end{align*}
\end{proof}

Next, we show that choosing $T=(n/\skewnessW)^{1+\alpha}$ for some $\alpha>0$ suffices to ensure that the optimization error is no larger than the statistical error.

\begin{corollary}\label{cor:HedgeOPOSuboptimalityBound}
    For any $T=\Omega(n/\skewnessW)$ and any $\wp\in\We$, the distribution $P_T$ satisfies
    \begin{equation*}
        \opsE_{\pi\sim P_T}\bra{\Rhatwp(\pi)}\le \minPi\maxwW\Rw(\pi) + \rateW.
    \end{equation*}
\end{corollary}
\begin{proof}
    Since $B\ge\abs{\Yi(a)}$, for any $\wp\in\We$,
    \begin{align*}
        \Rhatwp(\pi)&\le \Rwp(\pi)+\rateW \\
        &\le\Qwp(\pistarwp) - \Qwp(\pi) + \rateW \\
        &\le 2B + \rateW \\
        &\le 2B + \bigOp{\!\sqrtfracskewnessWn}.
    \end{align*}
    This last inequality gives a candidate bound for $\Rhatwp(\pi)$ to apply Proposition \ref{prop:HedgeOPOSuboptimalityBound}. Then, if we set $T=(n/\skewnessW)^{1+\alpha}$ for any choice of $\alpha>0$, we have that
    \begin{align*}
        2\cdot\rbra{2B+\bigOp{\!\sqrtfracskewnessWn}}\cdot\sqrt{\frac{\log|\We|}{T}}\le\littleop{\!\sqrtfracskewnessWn},
    \end{align*}
    and so
    \begin{align*}
        \opsE_{\pi\sim P_T}\bra{\Rhatwp(\pi)}\le\minPi\maxwW\Rw(\pi) + \rateW.
    \end{align*}
    
\end{proof}

\subsection{Proof of Theorem \ref{thm:MainTheorem}}

\MainTheorem*
\begin{proof}
    Fix $\epsilon>0$.
    For any $\wp\in\We$,
    \begin{align*}
        \Rwp(\pihat)&=\Deltawp(\pistarwp,\pihat) \\
        &=\Deltawp(\pistarwp,\pihat)-\Deltahatwp(\pistarwp,\pihat)+\Deltahatwp(\pistarwp,\pihat) \\
        &\le\supPiab\regabs{\Deltaw(\pia,\pib)-\Deltahatwp(\pia,\pib)}+\Deltahatwp(\pistarwp,\pihat).
    \end{align*}
    By Proposition \ref{prop:BoundingMixtureEmpiricalProcess}, the first term is bounded as
    \begin{align*}
        \supPiab\regabs{\Deltawp(\pia,\pib)-\Deltahatwp(\pia,\pib)}\le \rateW.
    \end{align*}

    For the second term, we have
    \begin{align*}
        \Deltahatwp(\pistarwp,\pihat)&=\Qhatwp(\pistarw)-\Qhatwp(\pihat) \\
        &\le\Qhatwp(\pihatwp)-\Qhatwp(\pihat) \\
        &=\Rhatwp(\pihat) \\
        &=\opsE_{\pi\sim P_T}\regbra{\Rhatwp(\pi)} \\
        &\le\minPi\maxwW\Rw(\pi) + \rateW,
    \end{align*}
    where the last inequality follows from the suboptimality bound of the HedgeOPO algorithm established in Corollary \ref{cor:HedgeOPOSuboptimalityBound}.

    Putting it all together, we have that for any $\wp\in\We$,
    \begin{align*}
        \Rwp(\pihat)\le\minPi\maxwW\Rw(\pi) + 2\rateW.
    \end{align*}

    Lastly, note that for any $\w''\in\W$, by construction there exists a $\w'\in\We$ such that $\norm{\w''-\w'}_1\le\eps$.
    Therefore, for any $\w''\in\W$,
    \begin{align*}
        R_{\w''}(\pihat)&=Q_{\w''}(\pistar_{\w''})-Q_{\w''}(\pihat) \\
        &=\sumiM\wi''\rbra{\Qi(\pistar_{\w''})-\Qi(\pihat)} \\
        &=\sumiM(\wi''-\wi')\rbra{\Qi(\pistar_{\w''})-\Qi(\pihat)}+\sumiM\wi'\rbra{\Qi(\pistar_{\w''})-\Qi(\pihat)} \\
        &\le 2B \sumiM(\wi''-\wi') + \bigrbra{Q_{\wp}(\pistar_{\w''}) - Q_{\w'}(\pihat)} \\
        &\le 2B \regnorm{\w''-\wp}_1 + \bigrbra{Q_{\wp}(\pistar_{\wp}) - Q_{\wp}(\pihat)} \\
        &\le 2B\epsilon + \Rwp(\pihat) \\
        &\le 2B\epsilon + \minPi\maxwW\Rw(\pi) + 2\rateW.
    \end{align*}

    Thus,
    \begin{align*}
        \maxwW\Rw(\pihat)&\le\minPi\maxwW\Rw(\pi) + 2B\epsilon + 2\rateW \\
        &=\minPi\maxwW\Rw(\pi) + 2B\epsilon + 2C\kappa(\Pi)\sqrt{V\cdot\fracskewnessWn\cdot\log\frac{\abs{\W_\epsilon}}{\delta}} + \littleop{\!\sqrtfracskewnessWn}.
    \end{align*}
    Redefining $C\mapsto C/2$ gives the result.
\end{proof}

\section{BOUNDING TARGET REGRET}\label{app:BoundingTargetRegret}

\subsection{Proof of Theorem \ref{thm:TargetTheorem}}

\TargetTheorem*

\begin{proof}
    First, we have that
    \begin{align*}
        R(\pihat)&=Q(\pistar) - Q(\pihat) \\
        &=Q(\pistar)-Q(\pihat)+\Qw(\pistar)-\Qw(\pistar) +\Qw(\pihat)-\Qw(\pihat) \\
        &=\bigrbra{Q(\pistar)-\Qw(\pistar)} + \bigrbra{\Qw(\pihat)-Q(\pihat)}+\Qw(\pistar)-\Qw(\pi) \\
        &\le 2\maxPi\abs{Q(\pi)-\Qw(\pi)} + \Qw(\pistar)-\Qw(\pihat) \\
        &\le 2\maxPi\abs{Q(\pi)-\Qw(\pi)} + \Qw(\pistarw)-\Qw(\pihat) \\
        &=2\maxPi\abs{Q(\pi)-\Qw(\pi)} + \Rw(\pihat).
    \end{align*}
    Since $R(\pihat)$ does not depend on $\w$, we can freely take the minimum over $\w\in\W$ on the first term on the right-hand side of this last inequality and the maximum over $\w\in\W$ on the second term. This gives
    \begin{align*}
        R(\pihat)&\le 2\minwW\maxPi\abs{Q(\pi)-\Qw(\pi)} + \maxwW\Rw(\pihat).
    \end{align*}

    Next, let $\calF_\infty^B$ be the space of functions uniformly bounded by $B$. 
    Since we assume the potential outcomes under the source distributions and the target distribution are bounded by $B$, it follows by the definition of integral probability metric distances under uniformly bounded test functions \citep{sriperumbudur2009integral} that
    \begin{align*}
        \maxPi\abs{Q(\pi)-\Qw(\pi)}&=\maxPi\Bigabs{\opsE_{Z\sim\calD}\bra{Y(\pi(X))}-\opsE_{Z\sim\calDw}\bra{Y(\pi(X))}} \\
        &\le\max_{f\in\calF_\infty^B}\Bigabs{\opsE_{Z\sim\calD}\bra{f(Z)}-\opsE_{Z\sim\calDw}\bra{f(Z)}} \\
        &\le B\cdot \TV(\calD,\calDw),
    \end{align*}
    and thus,
    \begin{align*}
        \disc(\calD,\calD_\W)\le B\cdot\minwW\TV(\calD,\calDw)=B\cdot\TV(\calD,\calD_\W).
    \end{align*}
\end{proof}

\section{ADDITIONAL ALGORITHM DETAILS}\label{app:AdditionalAlgorithmDetails}

\subsection{Oracle Implementation Details \& Complexity}
\label{app:AlgorithmComplexity}

At each time step, the EG-OPO algorithm requires $O(|\We|)$ calls to the OPO oracle, and the update at each time can be done in $O(|\We|)$ time. Therefore, the complexity of algorithm is $O(|\We|^2\cdot T\cdot C_{\text{OPO}})$ where $C_{\text{OPO}}$ is the complexity of the OPO oracle.

We used the policy class $\Pi$ of depth-2 decision trees and we used the PolicyTree method \citep{sverdrup2020policytree} as the OPO oracle for this class. This method does exact standard offline policy learning for this class of policies. It fits a fixed-depth decision tree by exhaustive search, given the set of rewards (AIPW scores) for all actions and the associated feature vectors. PolicyTree runs in $O(p^k\cdot n^k\cdot(\log n + d) + p\cdot n\cdot\log n)$ time, where $k$ is the depth of the decision trees, $p$ is the context dimension, $d$ is the number of actions, and $n$ is the number of samples.

A single 4.05 GHz Apple M3 Max CPU was used to run the experiments.

\subsection{Nuisance Parameter Estimation}\label{app:NuisanceParameterEstimation}

There is an additional upfront cost of policy value estimation.
Our results rely on efficient estimation of $\Qw(\pi)$ for any policy $\pi$, which in turn relies on efficient estimation of $\Qi(\pi)$. We leverage ideas of double machine learning \citep{chernozhukov2018double} to guarantee efficient policy value estimation given only high-level conditions on the predictive accuracy of machine learning methods on estimating the nuisance parameters of doubly robust policy value estimators. In this work, we use machine learning and cross-fitting strategies to estimate the nuisance parameters locally. The nuisance parameter estimates must satisfy the conditions of Assumption \ref{ass:FiniteSampleError}. Under these conditions, extensions of the results of \citep{chernozhukov2018double,athey2021policy} would imply that the doubly robust local policy value estimates $\Qhati(\pi)$ for any policy $\pi$ are asymptotically efficient for estimating $\Qi(\pi)$.

The conditions and estimators that guarantee these error assumptions have been extensively studied in the estimation literature. These include parametric or smoothness assumptions for non-parametric estimation. The conditional response function $\mui(x;a)=\E_{\calDcec}[\Yi(a)|\Xi=x]$ can be estimated by regressing observed rewards on observed contexts. The inverse conditional propensity function $\oi(x;a)=1/\P_{\calDcec}(\Ai=a|\Xi=x)$ can be estimated by estimating the conditional propensity function $\ei(x;a)=\P_{\calDcec}(\Ai=a|\Xi=x)$ and then taking the inverse. Under sufficient regularity and overlap assumptions, this gives accurate estimates. We can take any flexible approach to estimate these nuisance parameters. We could use standard parametric estimation methods like logistic regression and linear regression, or we could use non-parametric methods like classification and regression forests to make more conservative assumptions on the true models. Lastly, we note that if it is known that some sources have the same data-generating distribution, it should be possible to learn the nuisance parameters across similar sources.

In our experiments, we estimate the nuisance parameters with the non-parametric method of generalized random forests and we make use of the \texttt{grf} package \citep{athey2019generalized}.

\subsection{Cross-fitted AIPW Estimation}\label{app:CAIPWEstimation}

Once the nuisance parameters are estimated, they can be used for estimating AIPW scores. Refer to Algorithm \ref{alg:CAIPW} for the pseudocode on how we conduct the cross-fitting strategy for AIPW score estimation. Under this strategy, each source $\source\in\sourceSet$ divides their local dataset into $K$ folds, and for each fold $k$, the source estimates $\mui$ and $\oi$ using the rest $K-1$ folds. During AIPW estimation for a single data point, the nuisance parameter estimate that is used in the AIPW estimate is the one that was not trained on the fold that contained that data point. This cross-fitting estimation strategy is described in additional detail in \citep{zhou2023offline}.

\begin{algorithm}
    \caption{Cross-fitted AIPW: source-Side}
    \label{alg:CAIPW}
\begin{algorithmic}[1]
    \REQUIRE local data $\{(\Xij,\Aij,\Yij)\}_{i=1}^{\nc}$, nuisance parameter estimates $\muhati$ and $\ohati$, number of folds $K$
    \vspace{1pt}
    \STATE Partition local data into $K$ folds
    \STATE Define surjective mapping $k_c:[\nc]\to[K]$ of point index to corresponding fold index
    \FOR{$k=1,\dots,K$}
        \STATE Fit estimators $\muhati^{-k}$ and $\ohati^{-k}$ using rest of data not in fold $k$
    \ENDFOR

    \FOR{$i=1,\dots,\nc$}
        \FOR{$a\in\calA$}
            \STATE $\Ghatij(a)\leftarrow \muhatikij(\Xij;a)+\bigrbra{\Yij-\muhatikij(\Xij;a)}\cdot\ohatikij(\Xij;a)\cdot\ones\{\Aij=a\}$
        \ENDFOR
    \ENDFOR
\end{algorithmic}
\end{algorithm}

\section{ADDITIONAL DISCUSSION}\label{app:AdditionalDiscussion}

\subsection{Limitations \& Future Work}

Our work has several limitations that warrant further exploration. First, we make certain assumptions about the data-generating process that may not always hold. While we have discussed potential relaxations, such as relaxing the boundedness and uniform overlap assumptions in the data-generating distributions, further investigation is needed to fully understand the impact of these adjustments.
In particular, interesting question arises regarding the pessimism principle in overcoming the uniform overlap assumption \citep{jin2022policy}. Specifically, we wonder if it would be necessary to have coverage under the locally optimal policy for each data source, which could influence the robustness of our approach.

Moreover, in this work, we estimate nuisance parameters separately for each data source. When sources share the same data-generating distribution, there is an opportunity to improve efficiency by learning nuisance parameters across similar sources, rather than treating them independently.

Finally, we have not fully established whether the regret bounds we provide are optimal. Prior work \citep{mohri2019agnostic} suggests that skewness-based bounds for distributed supervised learning are optimal, and our results reduce to regret-optimal results from \citep{athey2021policy, zhou2023offline} when the sources are identical. This provides strong indications of optimality, but establishing lower bounds in our setting is an open area for future research.